%% file: arxiv.tex
\pdfoutput=1
\documentclass[sigconf]{acmart}

\copyrightyear{2026}
\acmYear{2026}
\setcopyright{cc}
\setcctype{by}
\acmConference[CIKM '26]{Proceedings of the 35th ACM International Conference on Information and Knowledge Management}{November 07--11, 2026}{Rome, Italy}
\acmBooktitle{Proceedings of the 35th ACM International Conference on Information and Knowledge Management (CIKM '26), November 07--11, 2026, Rome, Italy}
\acmDOI{10.1145/3799682.3841161}
\acmISBN{979-8-4007-2539-5/2026/11}

\usepackage{amsthm}
\usepackage{subcaption}
\usepackage{mathtools}
\usepackage{multirow}
\usepackage[capitalize,noabbrev]{cleveref}
\crefname{assumption}{Assumption}{Assumptions}
\Crefname{assumption}{Assumption}{Assumptions}
\crefname{proposition}{Proposition}{Propositions}
\Crefname{proposition}{Proposition}{Propositions}
\crefname{corollary}{Corollary}{Corollaries}
\Crefname{corollary}{Corollary}{Corollaries}

\usepackage{tikz}
\definecolor{bodyink}{HTML}{1F2937}
\definecolor{mutedtext}{HTML}{6B7280}
\definecolor{hairline}{HTML}{D1D5DB}
\definecolor{insightfill}{HTML}{FCFCFA}
\definecolor{catGood}{HTML}{3C6B49}
\definecolor{catBad}{HTML}{B5832A}
\definecolor{catOracle}{HTML}{2F5E8A}

\theoremstyle{plain}
\newtheorem{theorem}{Theorem}[section]
\newtheorem{proposition}[theorem]{Proposition}
\newtheorem{corollary}[theorem]{Corollary}
\theoremstyle{definition}
\newtheorem{assumption}[theorem]{Assumption}

\newcommand{\Pa}{\operatorname{Pa}}
\newcommand{\Ch}{\operatorname{Ch}}
\newcommand{\Sp}{\operatorname{Sp}}
\newcommand{\indep}{\mathrel{\perp\!\!\!\perp}}
\newcommand{\E}{\mathbb{E}}
\newcommand{\G}{\mathcal{G}}
\newcommand{\X}{\mathbf{X}}
\newcommand{\SCMBench}{\textsc{SCM3K}}
\newcommand{\RMSE}{\operatorname{RMSE}}

\microtypesetup{expansion=false}

\begin{document}

\title[The Good, the Bad, and the Ugly of MB]{The Good, the Bad, and
  the Ugly of Markov Boundary for Tabular Prediction}

\author{Shu Wan}
\email{swan@asu.edu}
\orcid{0000-0003-0725-3644}
\affiliation{%
  \institution{Arizona State University}
  \city{Tempe}
  \country{United States}
}

\author{Abhinav Gorantla}
\email{agorant2@asu.edu}
\orcid{0009-0003-1242-5671}
\affiliation{%
  \institution{Arizona State University}
  \city{Tempe}
  \country{United States}
}

\author{Huan Liu}
\email{huanliu@asu.edu}
\orcid{0000-0002-3264-7904}
\affiliation{%
  \institution{Arizona State University}
  \city{Tempe}
  \country{United States}
}

\author{K. Sel\c{c}uk Candan}
\email{candan@asu.edu}
\orcid{0000-0003-4977-6646}
\affiliation{%
  \institution{Arizona State University}
  \city{Tempe}
  \country{United States}
}

\renewcommand{\shortauthors}{Shu Wan, Abhinav Gorantla, Huan Liu, and K. Sel\c{c}uk Candan}



\begin{abstract}
Under standard graphical assumptions, the Markov boundary of a target
variable is the smallest set of features that renders every other
feature redundant. Once the boundary is observed, the target is
conditionally independent of the rest of the table. This is a tempting
object for tabular prediction, since it names exactly the columns a
model should need. Yet modern regressors are still trained on the
full feature set. We ask whether the Markov boundary is genuinely
useful for prediction on \SCMBench{}, a 3,450-task synthetic SCM
benchmark with feature counts from $40$ to $1000$ and six SCM
families, evaluated with six regressors.
The answer is more nuanced than the
theory suggests. Restricting a regressor to the oracle boundary often
improves prediction substantially, and the improvement grows as the
feature space becomes larger and sparser. But the natural pipeline of
recovering the boundary with causal discovery and training on the
recovered mask does not deliver. Existing estimators exhaust the
compute budget before reaching the regime where the boundary helps
most, and even where they run they rarely beat the full feature set.
We trace this to three causes. Discovery optimizes structural
recovery rather than prediction. False negatives and false positives
carry sharply asymmetric predictive cost. The exact boundary is
only one of many feature sets that beat all features.
We then develop
what these facts imply for prediction-aligned feature selection and for
tabular models that learn to use causal structure.
\end{abstract}

\maketitle

\section{Introduction}
\label{sec:intro}

The goal of tabular prediction is to estimate a target variable $Y$
from a table of candidate features. Two properties make a feature set
ideal for this task. It should be \emph{sufficient}. Conditioning on
it must preserve everything the table reveals about $Y$, so that no
predictive signal is discarded. It should also be \emph{minimal}. It
should hold nothing beyond what sufficiency demands, so that the
learner is not charged for redundant columns. Causal graphical models
give these two properties a single, precise solution. For a target
$Y$ in a directed graphical model, the Markov boundary $B(Y)$ consists
of its parents, its children, and the other parents of those children.
It is the unique smallest feature set for which $Y$ is conditionally
independent of every remaining feature once $B(Y)$ is observed
\citep{10.5555/534975,10.5555/1795555}. Under the standard Markov and
faithfulness assumptions \citep{10.5555/1642718,spirtes2000causation},
the boundary is at once minimal and sufficient.

Minimality and sufficiency, however, are not all a practitioner
needs. A selection procedure must also be \emph{scalable}. It
has to stay tractable as a table widens to hundreds or thousands of
columns. On this third axis the comfortable picture breaks. Classical
Markov boundary and causal discovery algorithms recover a minimal
sufficient set by design, but lean on independence search
whose cost grows steeply with the feature count
\citep{NIPS1999_5d79099f,aliferis2003hiton,tsamardinos2003algorithms}.
Models with implicit feature selection, such as the LASSO
\citep{10.1111/j.2517-6161.1996.tb02080.x} and tree-based ensembles
\citep{10.1145/2939672.2939785}, scale well and select sparse
predictors, but their selection follows marginal predictive
correlation rather than the conditioning structure, so the retained
set need not be sufficient. \Cref{fig:trinity} places
these families at the edges of a triangle whose corners are
scalability, minimality, and sufficiency. Each method reaches two
corners, and none reaches all three.

\begin{figure}[t]
  \centering
  \input{figs/src/trinity.tikz}
  \Description{Triangle with corners Sufficiency, Scalability, and
    Minimality. Tabular foundation models sit on the
    Scalability--Sufficiency edge, Markov boundary discovery on the
    Sufficiency--Minimality edge, and implicit feature selection on
    the Scalability--Minimality edge. The center is labeled Open
    design space.}
  \caption{Three desiderata of a feature-selection method for tabular prediction.}
  \label{fig:trinity}
\end{figure}

The remaining edge belongs to the regressors that increasingly define
tabular prediction. TabPFN and TabICL are transformers pre-trained on
millions of synthetic
prediction tasks. At test time they ingest an entire table in context
and predict without any per-dataset fitting
\citep{hollmann2022tabpfn,qu2025tabicl,muller2022transformers}. They
are fast, accurate, and scalable, yet by construction they consume
every column handed to them. Feature selection is simply not part of
the recipe. The model is trusted to discount whatever is irrelevant on
its own. Whether that trust is warranted is an empirical question.
Whether a regressor still gains from seeing only the columns that
matter is the question this paper studies.

We state the question plainly. \emph{Is the Markov boundary useful for tabular
prediction?} If $B(Y)$ is minimal and sufficient, a regressor
restricted to $B(Y)$ should lose no predictive information while
carrying far fewer columns. But the full table is sufficient too, so
population theory alone cannot answer the question. It is a
finite-sample claim about how a particular regressor reacts to
redundant columns, and it has to be measured.

We measure it with \SCMBench{}\footnote{\url{https://huggingface.co/datasets/CSE472-blanket-challenge/SCM3K}}, a controlled benchmark of 3,450
synthetic SCM tasks. It pairs Erd\H{o}s--R\'enyi DAGs with six SCM
families. It sweeps the
candidate feature count from $40$ to $1000$ and evaluates six
regressors. They include shrinkage baselines, a tree ensemble, a
neural network, and tabular foundation models. For every
task and regressor we compare the test error of training on all
features against training on the oracle Markov boundary, and call
their difference the \emph{MB gap}.

The story that emerges is more tangled than the theory predicts.
First, the boundary delivers. Restricting to $B(Y)$ improves
prediction for most regressors, and the MB gap widens steadily as the
feature space grows larger and sparser. Encouraged by this, we test
the obvious pipeline. We estimate the boundary with off-the-shelf
Markov-boundary and causal discovery, then train on the recovered
mask. The result disappoints. The estimators exhaust the compute budget long
before reaching the high-dimensional regime where the gap is largest,
and even where they do run, the recovered masks rarely beat the full
feature set. We then ask why, and find three reasons. Causal discovery
optimizes structural recovery, which is not the same objective as
prediction. Missing a boundary feature and adding a redundant one are
scored identically by recovery metrics, yet they carry sharply
asymmetric predictive cost. And the exact boundary is not the only
good answer. Many feature sets that differ from $B(Y)$ still beat the
full table. The three desiderata of \Cref{fig:trinity} pull against
one another exactly in the regime where prediction is hardest.

This is an inconvenient truth, but a generative one. Exact boundary
recovery is the objective inherited from causal discovery, but it is
the wrong target when the goal is prediction. The evidence points to
feature selection that is scalable, prediction-aligned, and willing to
trade strict
minimality for robustness. We close by developing these
implications
into concrete directions, from scaling boundary estimation through
amortized pre-training to co-learning the feature mask and the
predictor together.

\Cref{sec:background} sets up Markov boundary optimality and the SCM
benchmark. \Cref{sec:mb-gap} measures the oracle MB gap.
\Cref{sec:bad} tests the estimate-then-predict pipeline, and
\Cref{sec:ugly} dissects why it fails. \Cref{sec:beyond}
characterizes prediction-useful feature sets beyond the exact
boundary. \Cref{sec:implications} develops the implications into
research directions, \Cref{sec:related-work} situates the paper, and
\Cref{sec:conclusion} concludes.

\section{Preliminaries}
\label{sec:background}

\subsection{Markov boundary fundamentals}
\label{sec:mb-defs}

Let $\X=(X_1,\ldots,X_F)$ be the candidate features and let $Y$ be the
regression target. For $S\subseteq [F]$, write $\X_S$ for the
restricted feature vector and define the population squared-loss risk
\begin{equation}
R^*(S)=\E\!\left[(Y-\E[Y\mid \X_S])^2\right].
\label{eq:subset-risk}
\end{equation}
We call $S$ \emph{Bayes sufficient} when $R^*(S)=R^*([F])$.

For a DAG $\G$ over $\{Y,X_1,\ldots,X_F\}$, the graphical Markov
boundary of $Y$ is
\begin{equation}
B \equiv B(Y)=\Pa(Y)\cup\Ch(Y)\cup\Sp(Y),
\label{eq:boundary}
\end{equation}
where $\Sp(Y)$ are the other parents of $Y$'s children. We write
$k=|B|$, $\rho=k/F$ for the boundary fraction, and
\texttt{redundancy\_ratio}$=1-\rho$ in the empirical models.

\begin{assumption}[Graphical sufficiency and boundary non-degeneracy]
\label{ass:graphical}
The joint law of $(Y,\X)$ is Markov and faithful to a DAG $\G$ over
$Y$ and the features, and admits a density that is positive on a
product of the coordinate supports. The same DAG $\G$ governs the
training and test distributions, so the Markov boundary of $Y$ is invariant
across the train/test split. For every $j\in B(Y)$, the map
$x_j\mapsto\E[Y\mid \X_B=x_B]$ is non-constant on a set of positive
measure.
\end{assumption}

The product-support condition lets a conditional-mean identity rule
out a missing boundary coordinate; it holds for \SCMBench{}, whose
nodes carry independent full-support noise. The last condition asks
only that every boundary variable move the conditional mean of $Y$. It fails in the degenerate case where a
variable affects higher moments, such as the noise variance, without
affecting the conditional expectation.

\begin{theorem}[Boundary optimality]
\label{thm:mb-bayes}
Under \Cref{ass:graphical}, $B(Y)$ is Bayes sufficient and every
Bayes-sufficient set contains $B(Y)$. Hence $B(Y)$ is the unique
inclusion-minimal Bayes-sufficient feature set, and
\begin{equation}
R^*(B)=R^*([F]).
\end{equation}
\end{theorem}

\begin{proof}
The separator property gives $Y\indep \X_{[F]\setminus B}\mid \X_B$,
so $\E[Y\mid \X]=\E[Y\mid \X_B]$ and $R^*(B)=R^*([F])$. Since the
residual $Y-\E[Y\mid\X]$ is uncorrelated with every function of $\X$,
for any $S$
\begin{equation}
R^*(S)=R^*(B)
  +\E\!\left[\left(\E[Y\mid\X_B]-\E[Y\mid\X_S]\right)^{2}\right],
\label{eq:risk-gap}
\end{equation}
so $S$ is Bayes sufficient exactly when $\E[Y\mid\X_S]=\E[Y\mid\X_B]$
almost surely. Suppose such an $S$ omits some $j\in B$, and write
$g(\X_B)=\E[Y\mid\X_B]$. Then $g(\X_B)$ equals a function of $\X_S$
alone. Conditioning on $\X_S$, product support leaves $\X_{B\setminus S}$
with full conditional support, so $g$ must be constant in $x_j$,
contradicting \Cref{ass:graphical}. Hence $B\subseteq S$; since $B$ is
itself Bayes sufficient, it is the unique inclusion-minimal such set.
\end{proof}

\begin{corollary}[All features are also sufficient]
\label{cor:all-sufficient}
The full feature set $[F]$ is Bayes sufficient. Population sufficiency
alone therefore does not distinguish $B$ from $[F]$.
\end{corollary}

\subsection{SCM3K benchmark dataset}
\label{sec:dgp}

\SCMBench{} is a family of 3,450 synthetic SCM tasks.
Each split fixes
$F\in\{40,60,80,100,200,400,600,800,1000\}$, sets the DAG size to
$F+1$ nodes, and treats one node as the target. Low-dimensional splits
($F\le100$) use dense Erd\H{o}s--R\'enyi graphs with densities
0.2 and 0.4, with target MB-ratio band $[0.10,0.90]$.
Higher-dimensional splits use sparse graphs with densities 0.01,
0.02, and 0.04, with band $[0.05,0.95]$ \citep{erdos1960evolution}.

Each DAG is paired with six SCM families. They are linear Gaussian,
linear non-Gaussian, additive Gaussian, additive non-Gaussian,
post-nonlinear, and heteroskedastic.
Each task has $n=1000$ samples, \texttt{coeff\_range} $=1.0$,
and
\texttt{noise\_std} $=0.5$. Nodes are generated in topological order and
standardized after assignment.
Varsortability~\citep{reisach2021beware} and $R^2$-sortability~\citep{3666122.3666158} are addressed to avoid unexpected information leakage. The target
is selected by the MB-ratio band alone. Parentless targets may qualify,
but empty boundaries are excluded by the lower bound.
\Cref{tab:benchmark} summarizes the split design.

\begin{table}[t]
  \centering
  \caption{\SCMBench{} datasets statistics.}
  \label{tab:benchmark}
  \small
  \setlength{\tabcolsep}{4pt}
  \begin{tabular}{@{}lccc@{}}
    \toprule
    Splits & DAG density & MB-ratio band & Tasks per $F$ \\
    \midrule
    $F\le100$ & $\{0.2,0.4\}$ & $[0.10,0.90]$ & $300$ \\
    $F\ge200$ & $\{0.01,0.02,0.04\}$ & $[0.05,0.95]$ & $450$ \\
    \bottomrule
  \end{tabular}
\end{table}

\subsection{Prediction quantities}
\label{sec:notation}

For a fixed regressor $R$ and feature subset $S$, let
$\RMSE_R(S)$ be the test RMSE when $R$ is trained and queried using
only $\X_S$. The absolute and relative MB gaps are
\begin{align}
\Delta_{\mathrm{MB}}(R)
  &= \RMSE_R([F])-\RMSE_R(B), \label{eq:mb-gap-abs}\\
\delta_{\mathrm{MB}}(R)
  &= \Delta_{\mathrm{MB}}(R)/\RMSE_R([F]).
  \label{eq:mb-gap-rel}
\end{align}
Positive values mean that the oracle boundary improves prediction
over all features. For estimated masks $\widehat S$, we use the same
prediction gain scale.
\begin{equation}
\operatorname{prediction\_gain}(\widehat S)
  =\RMSE_R([F])-\RMSE_R(\widehat S).
\label{eq:prediction-gain}
\end{equation}
Precision, recall, false positives, and false negatives are always
computed against the oracle boundary $B$.

\begin{figure*}[htp]
  \centering
  \includegraphics[width=.85\textwidth]{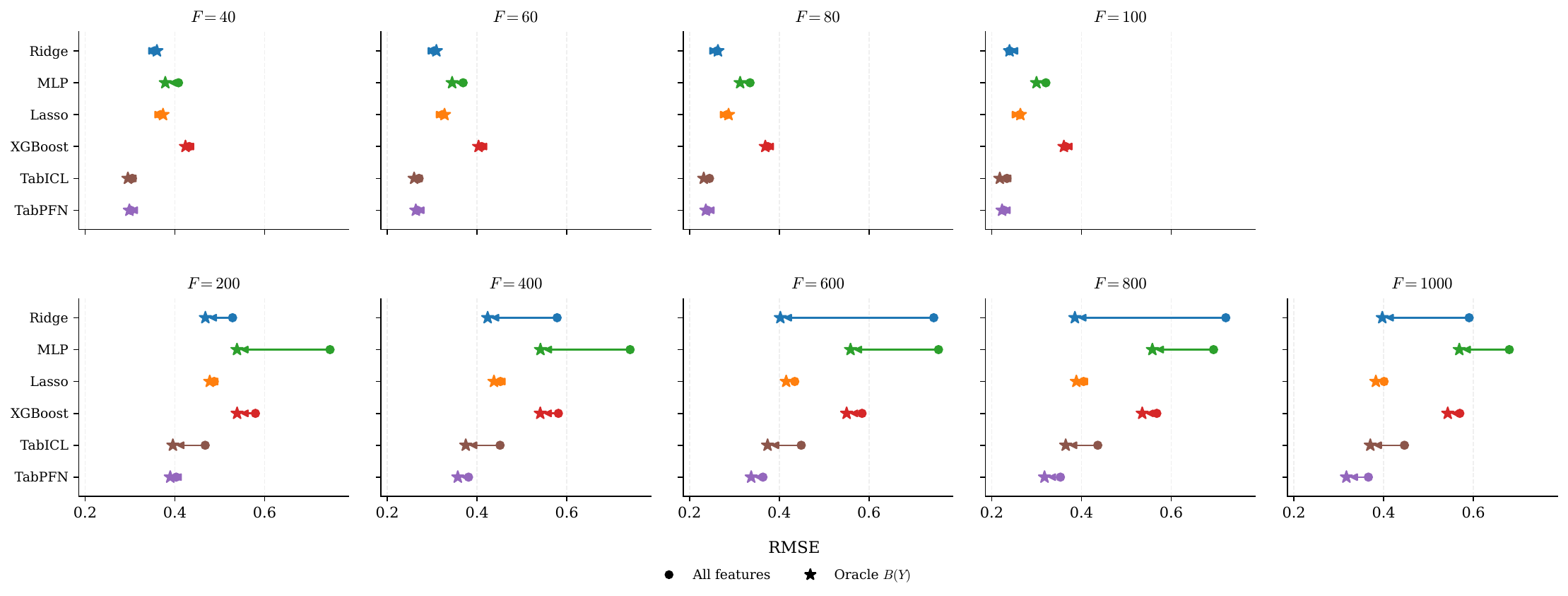}
  \caption{Oracle MB gap by downstream regressor. The
    boundary helps most for regressors that pay a high finite-sample
    cost for extra columns, and least for models with strong implicit
    feature selection.}
  \Description{Six-panel plot showing median relative MB gap with
    interquartile bands against feature count for Ridge, LASSO, MLP,
    XGBoost, TabICL, and TabPFN.}
  \label{fig:mb-gap-regressor}
\end{figure*}

\section{When the Boundary Helps}
\label{sec:mb-gap}

\noindent\emph{The Markov boundary is a useful prediction oracle. The
oracle gain is finite-sample and regressor-dependent. It grows when
redundant dimensions stress the downstream regressor, and it shrinks
when the regressor already performs feature selection.}

\subsection{Cross-regressor MB gap}
\label{sec:oracle-boundary}

We evaluate six regressors on \SCMBench{}. Ridge
and LASSO are shrinkage baselines
\citep{Hoerl01021970,10.1111/j.2517-6161.1996.tb02080.x}. MLP and
XGBoost cover nonlinear neural and tree-boosting regressors
\citep{HORNIK1989359,10.1145/2939672.2939785}. TabPFN and
TabICL represent prior-fitted tabular foundation regressors
\citep{hollmann2022tabpfn,qu2025tabicl,
qu2026tabiclv2}. Across these six regressors, the median
relative oracle RMSE reduction is strongly regressor-dependent.
The median reductions are Ridge $+35\%$, MLP $+24\%$, TabICL
$+18\%$, TabPFN $+12\%$, XGBoost $+4\%$, and LASSO $+2\%$.
The qualitative pattern is stable across the
$F$ sweep.
The gap is small below $F=200$ and in relatively dense graphs. It is
largest when the full feature set is high-dimensional and redundant.
\Cref{fig:mb-gap-regressor} reports this pattern across the six
regressors.

Implicit feature selection explains the small-gap end of the spectrum.
LASSO explicitly sparsifies. XGBoost selects split variables. Their
small oracle gaps therefore do not mean feature selection is
unimportant. They mean the feature-selection burden has already been
partly absorbed by the regressor. TabPFN and TabICL still show
nontrivial oracle gaps, including cases above a $10\%$ prediction-loss
change, so tabular foundation model regressors are not immune to
redundant features.

\subsection{MB gap attribution}
\label{sec:mixed-effect}

We summarize the cross-regressor effect with the final attribution
model on relative MB gap. Mixed-effect models are the standard
reference point for repeated or grouped observations
\citep{c0ae3670-c51a-3c43-9f8f-601a49c19723}. Here we report the fixed-effect
attribution form because the paper uses the model only to summarize
factor effects.
\begin{equation}
\begin{aligned}
\delta_{\mathrm{MB}}
  \sim\;&
  \texttt{redundancy\_ratio}
  + \log_{10}F
  + \texttt{scm\_family}
  + \texttt{regressor} \\
  &+ \texttt{regressor}:\texttt{redundancy\_ratio}
  + \texttt{regressor}:\log_{10}F.
\end{aligned}
\label{eq:mixed-effect}
\end{equation}
\Cref{fig:mixed-effect} visualizes the fitted marginal effects.

\begin{figure*}[t]
  \centering
  \includegraphics[
    width=.8\textwidth,
    keepaspectratio]{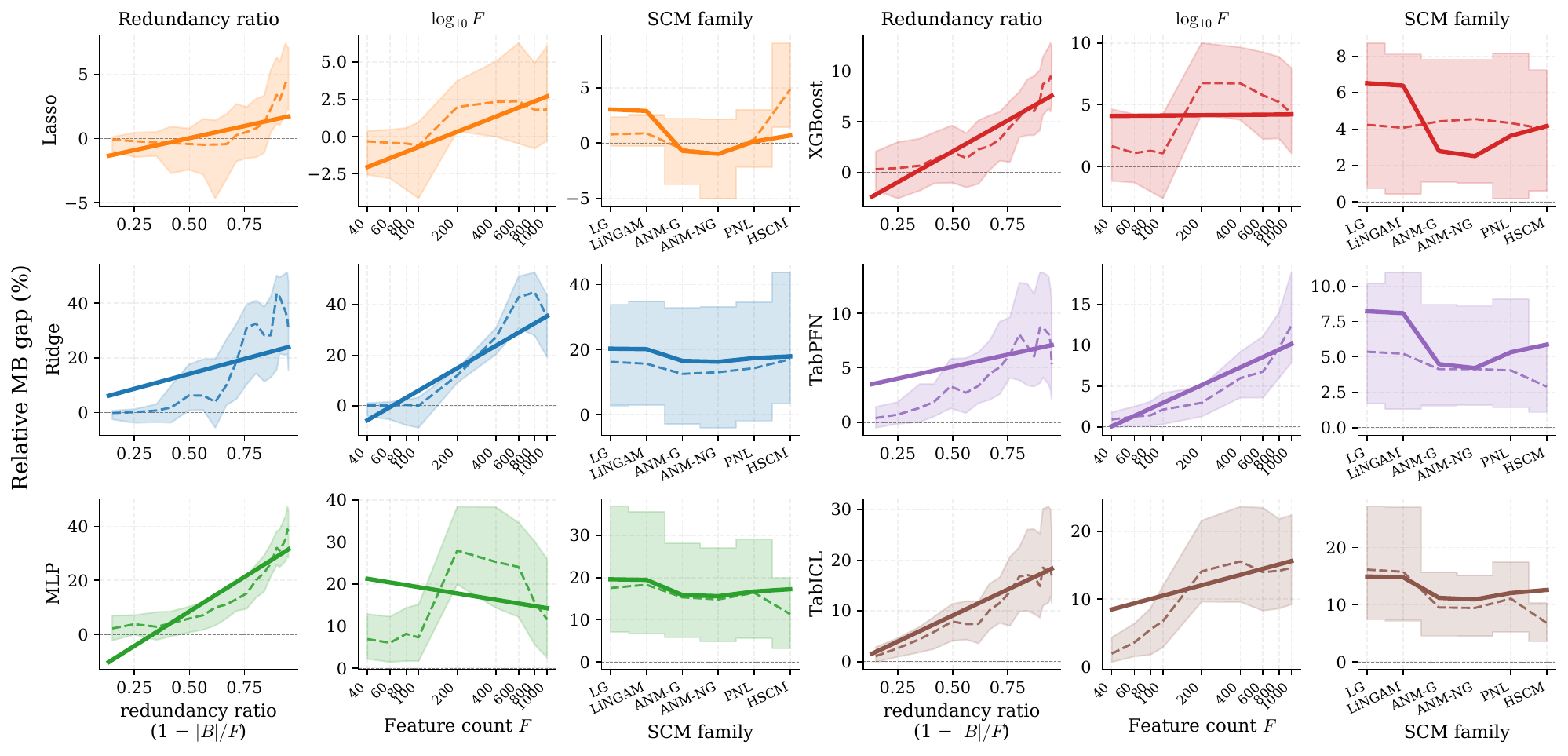}
  \caption{Regressor-conditioned attribution of the MB gap. The gap
    is explained primarily by redundancy, feature dimension, and
    regressor identity. SCM family contributes little after these
    factors are controlled.}
  \Description{Grid of marginal-effect plots from the final
    attribution model, with rows for regressors and columns for
    redundancy ratio, feature count, and SCM family.}
  \label{fig:mixed-effect}
\end{figure*}

The model confirms the visual pattern. Redundancy and regressor
identity are the strongest factors. Density is largely the same signal
as redundancy with the opposite sign. SCM family is weak after these
covariates are included. In univariate checks, redundancy ratio and
regressor identity are effectively tied for the largest adjusted
$R^2$ values ($0.221$ and $0.219$), followed by density ($0.187$) and
$\log_{10}F$ ($0.166$). SCM family explains much less ($0.012$).
This is the main empirical message of \Cref{sec:mb-gap}. The MB gap is
not a property of the graph alone, or of the regressor alone, but of
their interaction at a fixed sample budget.

\subsection{A finite-sample explanation}
\label{sec:good-theory}

The population theorem says that both $B$ and $[F]$ are sufficient. The
finite-sample question is how much variance the learner pays for using
unnecessary columns. A linear Gaussian SCM gives the simplest example
of this effect.

\begin{assumption}[Linear Gaussian working model]
\label{ass:finite}
For the purpose of this approximation only, suppose that the
conditional model is correctly specified and linear, the design is
Gaussian with $\X_S\sim\mathcal{N}(0,\Sigma_S)$ for a positive-definite
$\Sigma_S$ shared by the training and test distributions, the noise
variance is $\sigma^2$, and $n>p+1$ where $p=|S|$.
\end{assumption}

\begin{proposition}[Finite-sample MB gap]
\label{prop:finite-mb-gap}
Under \Cref{ass:finite}, for any Bayes-sufficient subset $S$ with
$p=|S|$,
\begin{equation}
\E\!\left[R(\widehat\beta_S)\right]
  = R^*(B)+\frac{\sigma^2 p}{n-p-1}.
\label{eq:finite-exact}
\end{equation}
In particular, the full-vs-boundary gap is
\begin{equation}
\frac{\sigma^2 F}{n-F-1}-\frac{\sigma^2 k}{n-k-1}
  = \frac{\sigma^2(F-k)}{n}+O\!\left(\frac{F^2+k^2}{n^2}\right).
\label{eq:finite-gap}
\end{equation}
\end{proposition}

\begin{proof}
For any Bayes-sufficient subset $S$, the linear conditional mean is
correctly specified in this model. The OLS estimator
$\widehat\beta_S$ is unbiased, so the excess test risk above $R^*(B)$
is pure estimation variance. Under Gaussian errors with a fixed
positive-definite design covariance, the prediction variance of OLS
on $p$ regressors with $n$ observations follows the Wishart
distribution, giving $\E[R(\widehat\beta_S)]=R^*(B)+\sigma^2
p/(n-p-1)$ \citep{hastie2009elements}. Both $B$ and $[F]$ are
sufficient, but they fit $k$ and $F$ coefficients respectively.
Subtracting gives the exact gap. A Taylor expansion in $1/n$ yields
the leading term $\sigma^2(F-k)/n$.
\end{proof}

The gap is the variance cost of estimating $F-k$ extra coefficients
whose population contribution is zero once $B$ is observed. This
expression is exact for OLS under the Gaussian model, but it is not
meant to cover every regressor. When $p$ approaches or exceeds $n$,
OLS is no longer the right working model. Regularized and nonlinear
regressors replace raw parameter count with effective complexity, and
their finite-sample cost for redundant features depends on the
specific inductive bias. That is exactly what \Cref{eq:mixed-effect}
captures through regressor-specific slopes.

\section{The Emperor's New Blanket}
\label{sec:bad}

\noindent\emph{Oracle boundaries can improve prediction, but current
causal estimators rarely turn that advantage into wins.}

The natural recipe is simple. Estimate the Markov boundary, select the feature set according to the estimated boundary, and train the same downstream
regressor. We run the experiments on Sol \cite{sol} instances with three causal
estimators. GES is a score-search method
\citep{10.1162/153244303321897717}. Grow-Shrink is a local boundary
estimator \citep{NIPS1999_5d79099f}. HITON-MB follows
the same local discovery tradition \citep{aliferis2003hiton}.

\begin{figure}[htp]
  \centering
  \includegraphics[width=0.9\linewidth]{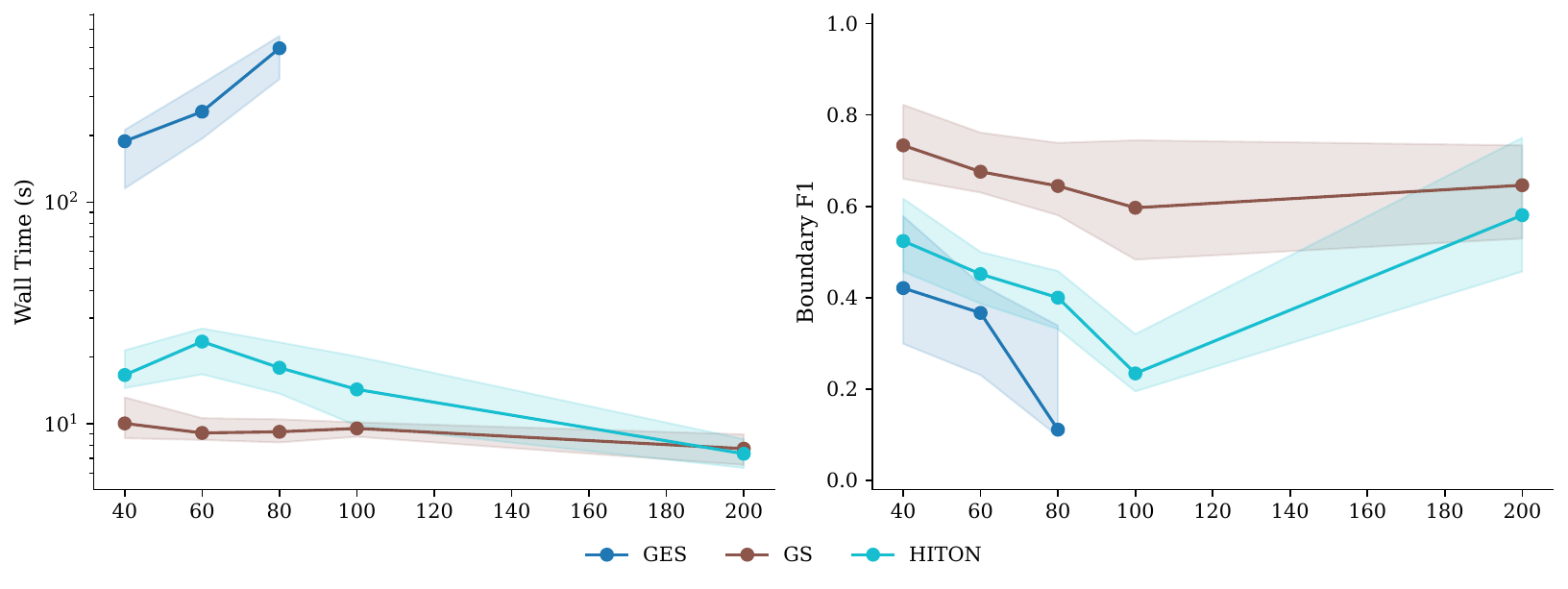}
  \caption{Causal boundary estimators do not scale well when the
    oracle MB gap is largest.}
  \Description{Two-panel plot showing wall time on a logarithmic scale
    and boundary F1 versus feature count for GES, Grow-Shrink, and
    HITON-MB.}
  \label{fig:runtime-vs-f}
\end{figure}

\begin{table*}[t]
  \centering
  \caption{Causal Markov-boundary estimators are accurate only in the
    low- to mid-dimensional regime and do not reliably improve
    prediction.}
  \label{tab:causal-scale}
  \normalsize
  \setlength{\tabcolsep}{6.5pt}
  \renewcommand{\arraystretch}{1.05}
  \begin{tabular}{@{}llcccccccc@{}}
    \toprule
    $F$ & Method
      & F1 & Prec. & Recall
      & Win/Ridge & Win/TabPFN & Win/XGB
      & Time (s) & Completion \\
    \midrule
    \multirow{3}{*}{$40$}
      & GES         & 0.438 & 0.826 & 0.330 & 6.7\%  & 1.1\%  & 13.6\% & 162.1 & 98.9\% \\
      & Grow-Shrink & 0.737 & 0.844 & 0.691 & 48.6\% & 20.7\% & 48.6\% & 10.9  & 99.4\% \\
      & HITON-MB    & 0.523 & 0.764 & 0.432 & 10.5\% & 5.0\%  & 25.3\% & 18.0  & 100.0\% \\
    \midrule
    \multirow{3}{*}{$60$}
      & GES         & 0.339 & 0.875 & 0.228 & 4.3\%  & 0.0\%  & 9.2\%  & 278.7 & 76.7\% \\
      & Grow-Shrink & 0.688 & 0.861 & 0.606 & 55.3\% & 16.2\% & 46.4\% & 9.5   & 99.4\% \\
      & HITON-MB    & 0.451 & 0.791 & 0.343 & 6.9\%  & 0.0\%  & 19.8\% & 22.1  & 98.3\% \\
    \midrule
    \multirow{3}{*}{$80$}
      & GES         & 0.214 & 0.824 & 0.138 & 0.0\%  & 0.0\%  & 0.0\%  & 475.1 & 9.4\% \\
      & Grow-Shrink & 0.652 & 0.820 & 0.585 & 58.7\% & 22.2\% & 51.7\% & 9.3   & 100.0\% \\
      & HITON-MB    & 0.392 & 0.734 & 0.287 & 5.2\%  & 0.0\%  & 28.6\% & 19.0  & 97.8\% \\
    \midrule
    \multirow{3}{*}{$100$}
      & GES         & $-$   & $-$   & $-$   & $-$    & $-$    & $-$    & $-$   & - \\
      & Grow-Shrink & 0.623 & 0.711 & 0.610 & 35.4\% & 39.1\% & 43.8\% & 9.4   & 99.4\% \\
      & HITON-MB    & 0.253 & 0.576 & 0.180 & 1.7\%  & 0.0\%  & 8.3\%  & 16.2  & 100.0\% \\
    \midrule
    \multirow{3}{*}{$200$}
      & GES         & $-$   & $-$   & $-$   & $-$    & $-$    & $-$    & $-$   & - \\
      & Grow-Shrink & 0.633 & 0.570 & 0.752 & 98.9\% & 37.4\% & 77.9\% & 8.1   & 100.0\% \\
      & HITON-MB    & 0.588 & 0.818 & 0.485 & 58.0\% & 1.1\%  & 55.8\% & 8.0   & 98.9\% \\
    \bottomrule
  \end{tabular}
\end{table*}

The recovery results are not enough to support the oracle
story.
Grow-Shrink is the most reliable of the three up to $F=200$. HITON-MB
has high precision but low recall. GES has high precision but
very
low recall and becomes unusable past $F=80$ under the \SCMBench{}
budget. GES completion falls to $9.4\%$ at $F=80$ and $0\%$ at
$F=100$. Grow-Shrink reaches $F=200$ with mean F1 $0.633$, while
HITON-MB reaches $F=200$ with mean F1 $0.588$ but recall only
$0.485$. These are precisely the feature sizes below or near the
region where the oracle MB gap is still limited.
\Cref{fig:runtime-vs-f,tab:causal-scale} show the resulting scale
mismatch.

The same table also shows a precision-oriented failure mode. Causal
recovery methods, especially local constraint-based methods, often
prefer conservative blankets with higher precision than recall.
HITON-MB has higher precision than recall at every evaluated feature
count, and Grow-Shrink is precision-heavy through $F=100$ before
flipping at $F=200$. This matters for prediction because the next
section shows that false negatives are usually more costly than false
positives.

The downstream prediction outcome also depends on the
regressor. Ridge
benefits most from masks because it has little built-in feature
selection. XGBoost benefits less because split selection already
filters many irrelevant variables. TabPFN often loses under estimated
masks, despite its nonzero oracle MB gap, indicating that noisy masks
can remove useful context even when the true boundary would help. For
example, Grow-Shrink reaches $98.9\%$ wins for
Ridge at $F=200$, but only $37.4\%$ for TabPFN at the same split.
HITON-MB wins at most $5.0\%$ of TabPFN tasks.
The paired recovery and win-rate columns in \Cref{tab:causal-scale}
show this failure directly.
This section stops at the empirical failure. The next section explains
why the failure is structural.

\section{Failure Mechanisms}
\label{sec:ugly}

\subsection{Scalability}
\label{sec:ugly-scale}

\noindent\emph{Causal estimators struggle to reach the regime where
the oracle MB gap is largest.}

The scale problem is not merely implementation overhead. Local
constraint-based methods perform conditional-independence tests over
candidate neighborhoods. In the worst case, conditioning sets up to
size $d$ induce $O(F^d)$ candidate tests. This is the same
combinatorial pressure that appears throughout constraint-based
structure learning
\citep{tsamardinos2003algorithms,aliferis2010local,tsamardinos2006max,
yu2020causality}.
Score-search methods avoid the same test enumeration but still search
over a rapidly growing graph space.

Empirically, this complexity meets \SCMBench{} exactly where it
hurts. GES is effectively limited to $F\le80$ under the run
budget,
while Grow-Shrink and HITON-MB are capped around $F=200$. \Cref{sec:mb-gap}
showed that the oracle gap is small below $F=200$ and grows
in higher-dimensional redundant settings. The available causal
estimators therefore operate mostly where the prediction reward is
weakest.

\subsection{Asymmetric loss}
\label{sec:ugly-asym}

\noindent\emph{Exact-boundary metrics treat false positives and false
negatives as identification errors, but prediction risk weights them
very differently.}

We perturb the oracle boundary directly to isolate the two error
types. A false negative removes a true boundary feature. A false
positive keeps the boundary intact but adds a non-boundary feature.
Across regressors, the model-free per-feature cost ratio
$\alpha_{\mathrm{FN}}/\alpha_{\mathrm{FP}}$ is greater than one in
every reported cell. The magnitude is regressor-dependent. LASSO
ratios are large because the FP cost is nearly zero. Ridge ratios
shrink with $F$ as extra columns become costly. TabPFN gives the
canonical asymmetric pattern. It has large FN cost and moderate FP cost.
\Cref{fig:fn-fp} summarizes the perturbation experiment.

\begin{figure*}[htp]
  \centering
  \includegraphics[width=.75\textwidth]{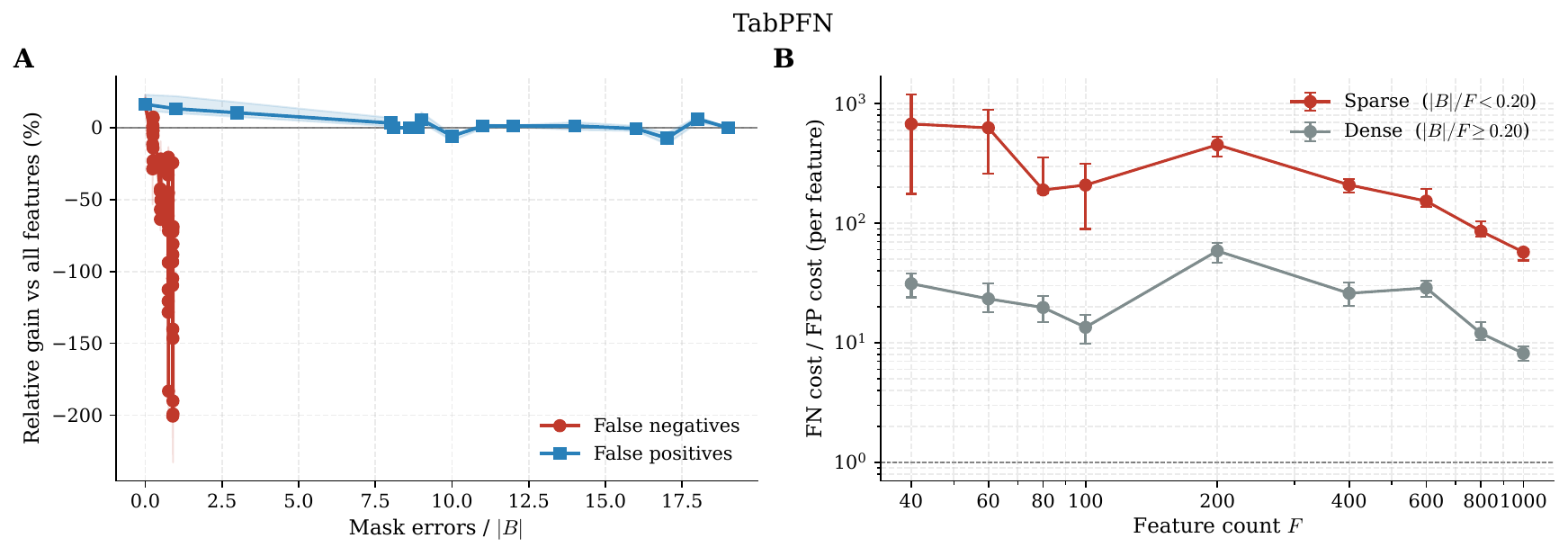}
  \caption{False negatives and false positives have different
    downstream costs. The TabPFN panel shown here follows the same
    qualitative asymmetry as the other regressors.}
  \Description{Two-panel figure showing false-negative and
    false-positive perturbation costs and per-feature cost ratios.}
  \label{fig:fn-fp}
\end{figure*}

The asymmetry is the omitted-variable-bias story in predictive form.
Let $\widehat S$ be an estimated mask,
$\Delta^+=\widehat S\setminus B$, and
$\Delta^-=B\setminus\widehat S$.

\begin{proposition}[Mask-error decomposition]
\label{prop:mask-error}
In the linear Gaussian example of \Cref{ass:finite}, adding
$s^+=|\Delta^+|$ redundant variables to a mask that contains $B$
preserves the population predictor. Its leading cost is an additional
variance term of order $\sigma^2 s^+/n$. Omitting boundary variables
introduces a population-risk term that does not vanish as $n$ grows.
Write $M=\widehat S$ and $O=B\setminus M$; then
\begin{equation}
R^*(M)=R^*(B)+
  \E\!\left[\left(
    \X_O^\top\beta_O
    -\E[\X_O^\top\beta_O\mid \X_M]
  \right)^{\!2}\right].
\label{eq:ovb}
\end{equation}
\end{proposition}

\begin{proof}
If $\widehat S\supseteq B$, the Markov-boundary separator property
gives $Y\indep \X_{[F]\setminus \widehat S}\mid \X_{\widehat S}$.
Consequently,
$\E[Y\mid \X_{\widehat S}]=\E[Y\mid \X_B]$ in population. The extra
coordinates in $\Delta^+$ have zero population contribution once $B$
is conditioned on. They can still increase finite-sample error because
their coefficients must be estimated, giving the leading
$\sigma^2 s^+/n$ variance cost under the linear Gaussian example.

Now suppose $\widehat S$ omits at least one boundary variable. Write
$M=\widehat S$ and $O=B\setminus M$. In the linear Gaussian example,
the target can be written as
\[
Y=\X_M^\top\beta_M+\X_O^\top\beta_O+\varepsilon,
\qquad
\E[\varepsilon\mid \X_M,\X_O]=0.
\]
The best predictor from $\X_M$ alone is
$\E[Y\mid\X_M]=\X_M^\top\beta_M+\E[\X_O^\top\beta_O\mid\X_M]$.
The residual $\X_O^\top\beta_O-\E[\X_O^\top\beta_O\mid\X_M]$ is
the component of the omitted signal that $\X_M$ cannot recover, and
its expected squared magnitude is the population-risk penalty in
\Cref{eq:ovb}. This term does not vanish with more samples
\citep{hastie2009elements,0f8094d1-2ebf-3360-90f1-98e50c172f31}.
\end{proof}

Thus a high-precision, low-recall boundary estimator can look
reasonable under exact-recovery metrics while still damaging
prediction. This explains why F1, SHD, or constraint satisfaction are
incomplete objectives for tabular prediction.

\subsection{Minimality}
\label{sec:ugly-family}

\noindent\emph{The exact Markov boundary is the minimal sufficient
set, not the only prediction-relevant set.}

Theoretical optimality singles out $B$ because it is minimal. For
prediction, however, minimality is only one part of the story. Any
controlled superset of $B$ preserves the population conditional and
may be preferable to a brittle estimate that misses boundary
variables.

\begin{proposition}[Superset sufficiency]
\label{prop:superset}
If $S\supseteq B$, then $S$ is a Markov blanket of $Y$ and
$R^*(S)=R^*(B)$.
\end{proposition}

\begin{proof}
Write $C=S\setminus B$ and $T=[F]\setminus S$. The graphical
Markov-boundary property gives $Y\indep (X_C,X_T)\mid X_B$. By weak
union, $Y\indep X_T\mid X_B,X_C$, so $S$ is a Markov blanket and
$\E[Y\mid\X_S]=\E[Y\mid\X_B]$. The risk equality
$R^*(S)=R^*(B)$ follows directly.
\end{proof}

This proposition does not claim that larger masks beat the oracle
boundary. It says that exact recovery is too narrow as a training or
evaluation target. Over-inclusive masks may lose finite-sample
efficiency, but they can still be much safer than masks that drop
boundary variables. This motivates the prediction-aligned view in the
next section.

\section{Beyond the Boundary}
\label{sec:beyond}

The previous sections explain why exact MB recovery is a poor proxy
for prediction. This raises the question the rest of the paper turns
on. If the exact boundary is not the right target, what makes a
feature set \emph{good} and \emph{useful} for prediction? We answer it
constructively with two computation-only devices. The first is layered
blankets. The second is a prediction gain map over mask precision and
recall.
We then state the prediction-aligned object directly.

\subsection{Alternative blankets}
\label{sec:layered}

Let $B(v)$ denote the Markov boundary of node $v$ in the full DAG.
Define accumulated layers
\begin{equation}
L_{\le 1}=B(Y),\qquad
L_{\le k+1}
  =
  \left(L_{\le k}\cup\bigcup_{v\in L_{\le k}}B(v)\right)\setminus\{Y\}.
\label{eq:layered}
\end{equation}
The associated blanket rank $\kappa(v)$ is the first $k$ such that
$v\in L_{\le k}$, with disconnected nodes assigned an infinite rank.
The comparison baseline is target proximity, the shortest-path
distance from $Y$ in the undirected skeleton.
\Cref{fig:dag-layered-vs-proximity} contrasts these two expansions.
Layered blankets follow Markov-boundary closure, while target
proximity follows graph distance from $Y$. The two already differ at
rank one, and the difference is exactly the Markov property.
Layered@1 is the Markov boundary $B(Y)$ itself. Proximity@1 collects
only the immediate skeleton neighbors of $Y$, namely its parents and
children. A spouse that is not itself a parent or child of $Y$ lies
two hops away in the skeleton, reachable only through a shared child,
so proximity@1 omits it and is not a Markov blanket. The comparison
that follows isolates the predictive consequence of dropping those
spouses, that is, of breaking the Markov property by one step.

\begin{figure*}[t]
  \centering
  \includegraphics[width=0.7\textwidth]{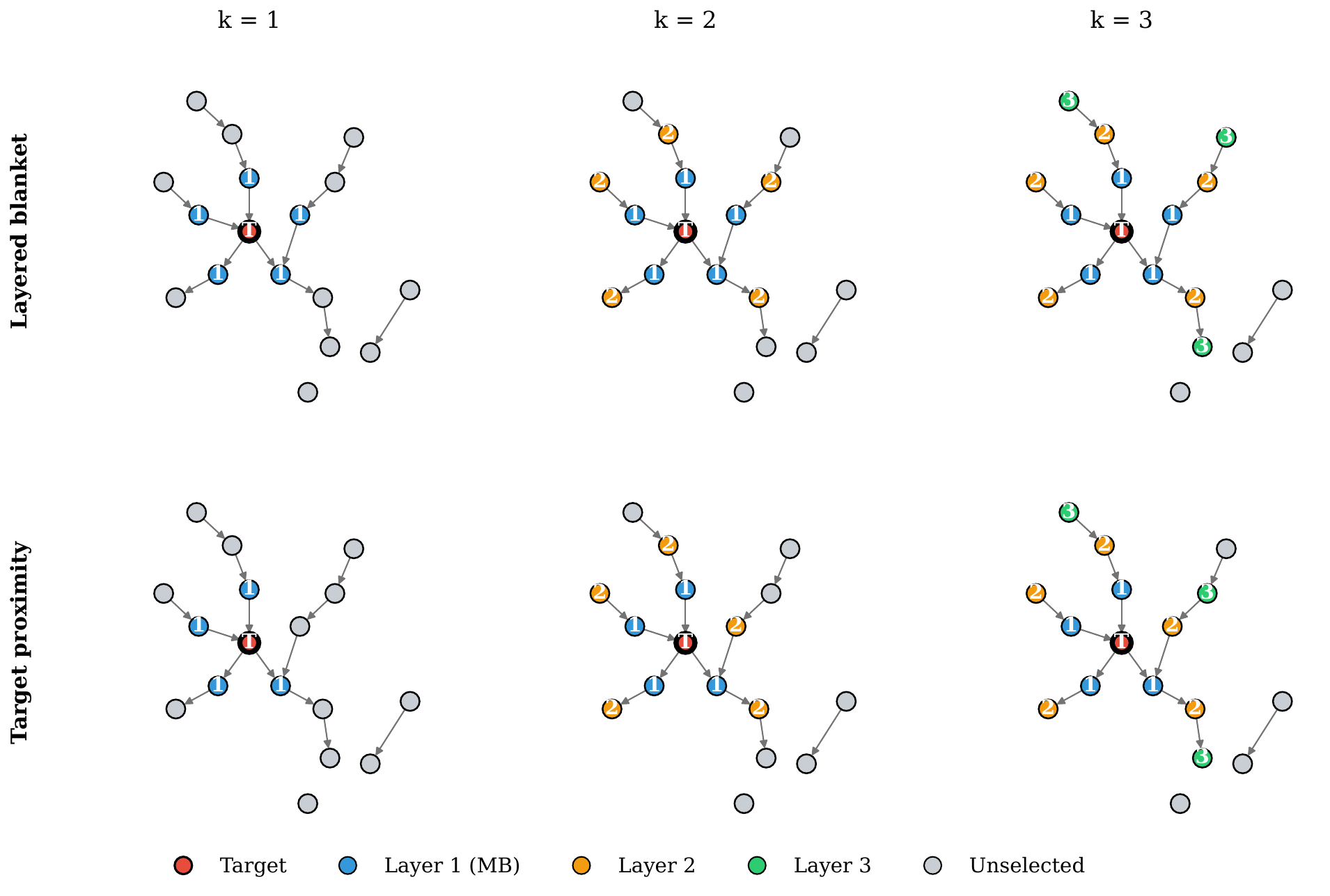}
  \caption{Comparison of two alternative blankets.
    The top row expands by accumulated Markov-boundary layers, while
    the bottom row expands by shortest-path distance from the target.}
  \Description{A two-by-three DAG schematic comparing layered
    blanket masks and target-proximity masks at ranks one, two, and
    three, with nodes colored by Markov-boundary layer.}
  \label{fig:dag-layered-vs-proximity}
\end{figure*}

Every accumulated layer contains $B(Y)$ and is therefore a Markov
blanket of $Y$ by \Cref{prop:superset}; the construction is a
computational device, not a new structural result.

Empirically, $L_{\le1}=B(Y)$ is the peak for every $F$, as expected.
It is the oracle boundary. Later layers dilute the oracle. The value
of the layered construction is not that larger layers outperform
$B$, but that they expose structured over-inclusive masks and separate
boundary-aware expansion from plain graph distance. For TabPFN,
layered@1 is substantially better than target-proximity@1 for
$F=200$--$800$. The margin ranges from $+10.8$ to $+20.9$ percentage
points over this high-dimensional band. Ridge shows the same
boundary-aware advantage at every $F$ split, with margins from
$+9.9$ to $+17.0$ percentage points. Because layered@1 and proximity@1
differ in that proximity can omit the non-adjacent spouses of $Y$,
this margin isolates the empirical predictive consequence of excluding
them from the mask. A mask that breaks the Markov property by one
structural step gives up a large, regressor-independent amount of
accuracy. Preserving the Markov property means including spouses.
That is what makes a feature set useful, not mere proximity.
\Cref{fig:layered-curve-abs} reports the layered comparison in
absolute RMSE for TabPFN.

\begin{figure}[t]
  \centering
  \includegraphics[width=.8\columnwidth]{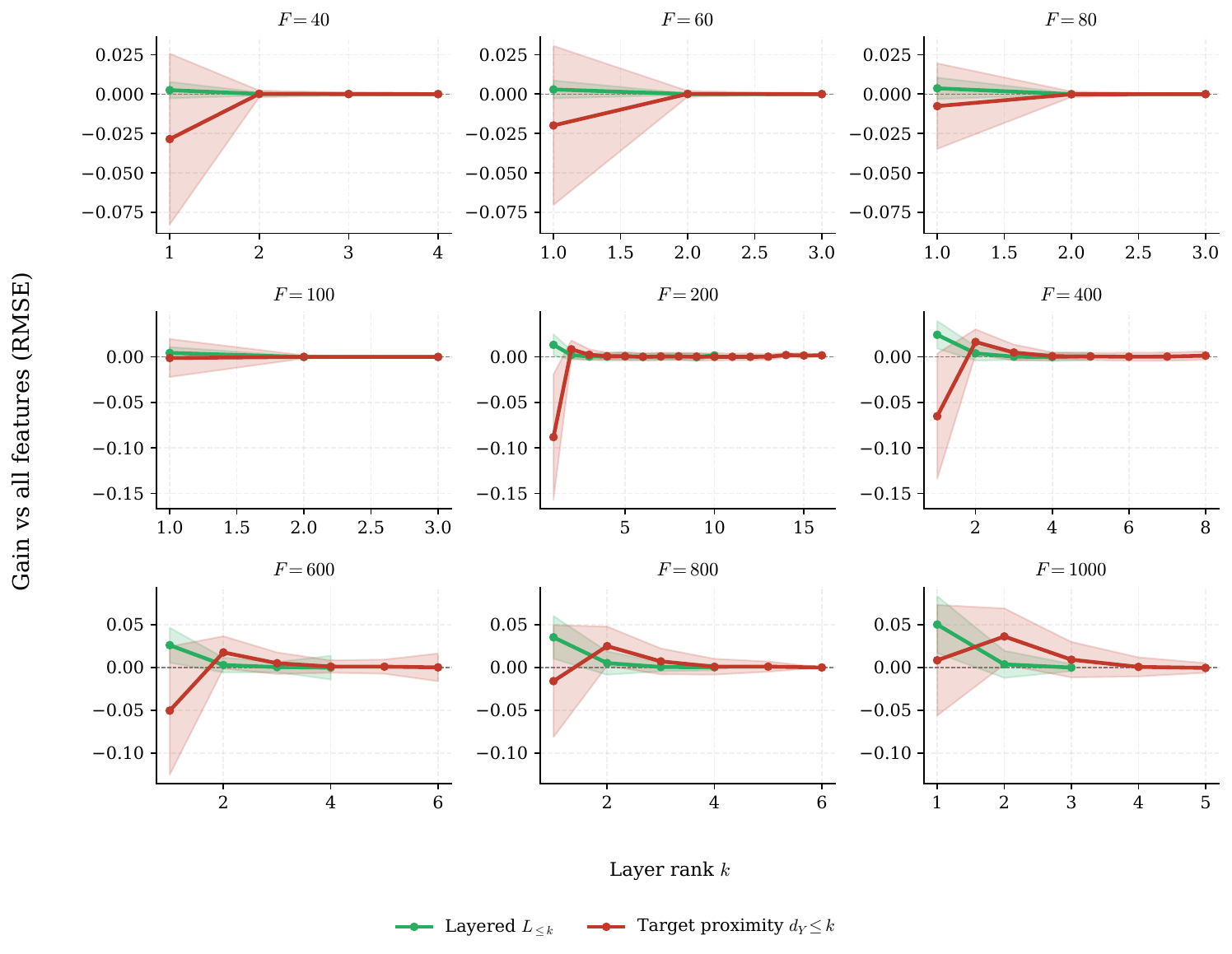}
  \caption{Effect of two alternative blankets on predictions}
  \Description{Layered-blanket curve for TabPFN in absolute RMSE,
    comparing accumulated Markov-boundary layers against
    target-proximity masks.}
  \label{fig:layered-curve-abs}
\end{figure}

\subsection{Prediction gain maps}
\label{sec:reward-map}

\Cref{sec:layered} established that the exact boundary is not the only
feature set worth having. Layered blankets are strict supersets
of $B(Y)$, and they also beat the full feature set. That turns a yes/no
question into a quantitative one. Across the whole space of candidate
masks, which ones are good, and what property makes them good? We go
one step beyond \Cref{sec:layered} and characterize good masks
directly, rather than enumerate them.

To keep the characterization clean we hold two things fixed. We
condition on a single feature count $F$ and a single downstream
regressor, so that the only quantity varying is the \emph{composition}
of the mask itself. Each panel of \Cref{fig:prediction-gain-map} simply repeats
the analysis for one regressor. A mask is then described by how it
differs from the oracle boundary. Its true positives (boundary
features kept), false negatives (boundary features dropped), and false
positives (non-boundary features added). On controlled perturbations
of $B$ we fit a simple model of prediction gain,
\begin{equation}
\operatorname{prediction\_gain}
  \sim
  \texttt{tp\_count}
  + \texttt{fn\_count}
  + \texttt{fp\_count}/n,
\label{eq:prediction-gain-model}
\end{equation}
the three coefficients answer the question directly. A good blanket is
one whose true-positive, false-negative, and false-positive counts
place it on the winning side of the fitted model. The map is a
descriptive, regressor-conditioned approximation, not a mask-risk law:
counts discard feature identity, so two masks with the same
$(\texttt{tp},\texttt{fn},\texttt{fp})$ can carry very different risk.

To read the model in the more familiar precision--recall coordinates,
note that at a precision--recall pair $(\pi,r)$ the implied counts are
\begin{align}
\texttt{fn} &= (1-r)|B|,\\
\texttt{fp} &= r|B|(1/\pi-1).
\end{align}
Substituting these into \Cref{eq:prediction-gain-model} turns it into a map
over $(\pi,r)$. Masks above the zero contour are predicted to beat all
features, and masks below it are not. \Cref{fig:prediction-gain-map} draws this
\emph{prediction-useful region}. It is a concrete description of what
a good blanket looks like for a given regressor.

The two overlaid trajectories place the mask families from
\Cref{sec:layered} on this plane. The blue layered path starts at the
top-right oracle point. It stays on the recall-one edge because every
layered mask contains $B(Y)$, so increasing $k$ only adds variables
outside the boundary. The brown proximity path starts from direct
skeleton neighbors of $Y$. That first mask excludes spouses, so recall
is low and the predicted gain is poor. Larger radii repair the false
negatives, but they also add many false positives. This mirrors the
actual layered-curve results. Layered@1 is the oracle peak, while
proximity@1 underperforms in the same regimes. The prediction gain map
is therefore a useful local approximation of the empirical mask-family
landscape, not only a fitted contour.

\begin{figure*}[t]
  \centering
  \begin{subfigure}[t]{0.4\textwidth}
    \centering
    \includegraphics[width=\linewidth]{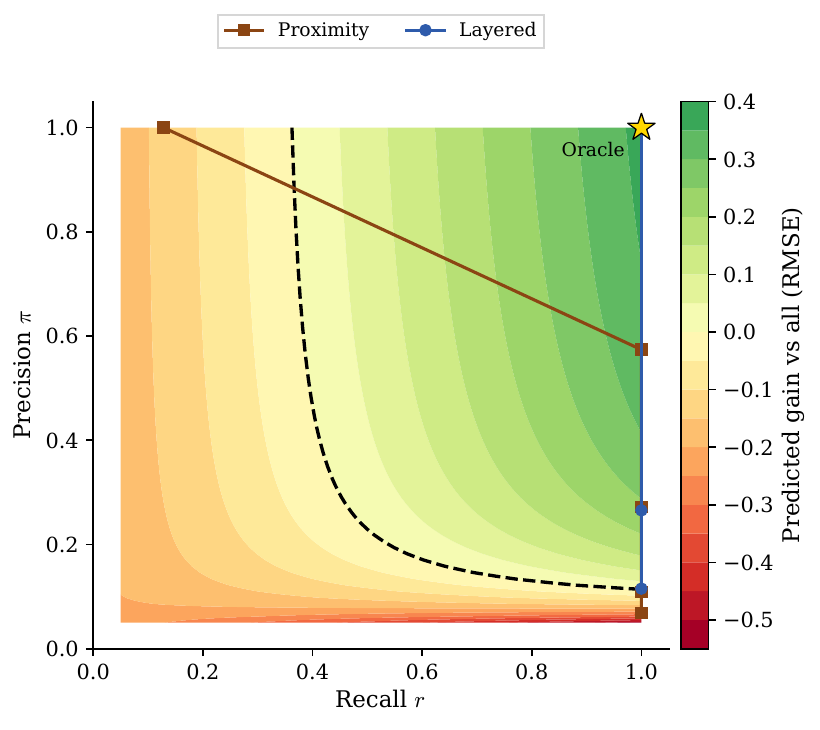}
    \caption{Ridge}
  \end{subfigure}
  \hfill
  \begin{subfigure}[t]{0.4\textwidth}
    \centering
    \includegraphics[width=\linewidth]{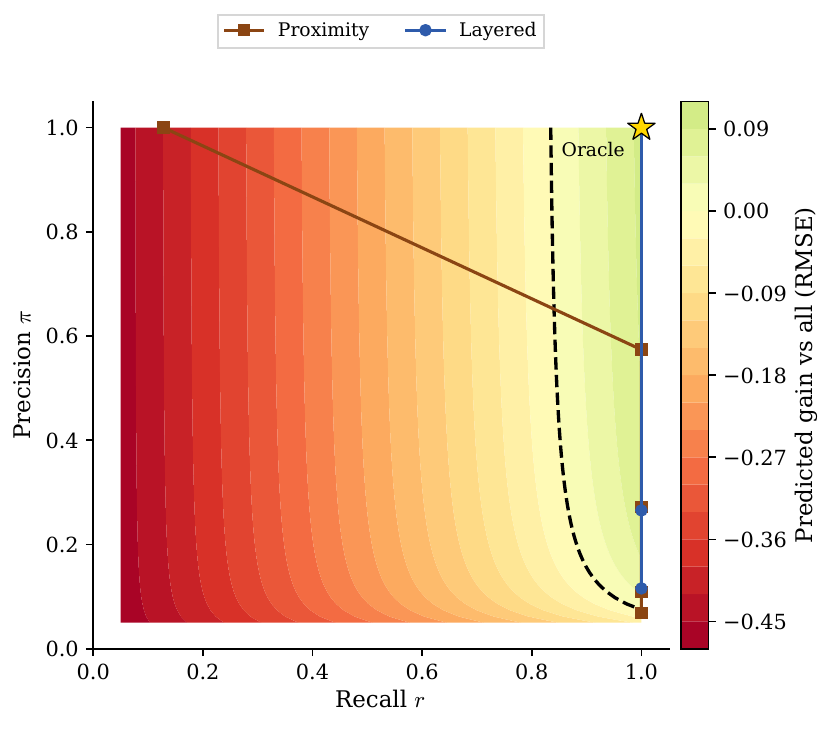}
    \caption{TabPFN}
  \end{subfigure}
  \caption{The dashed line marks the boundary between masks that improve prediction performance and those that hurt it.}
  \Description{Two prediction gain maps over precision and recall, one for
    Ridge and one for TabPFN. The zero contour marks the transition
    between masks predicted to help and masks predicted to hurt
    relative to all features. Brown trajectories show target-proximity
    masks, while blue trajectories show layered-blanket masks.}
  \label{fig:prediction-gain-map}
\end{figure*}

The region is wide, and it is not the same for every regressor. The
fitted TabPFN map has $R^2=0.758$ with a strong negative coefficient
for missed boundary features. The Ridge map has lower $R^2=0.436$ and
a much larger false-positive penalty, reflecting Ridge's sensitivity
to redundant columns. The prediction gain map also makes the
false-positive and false-negative tradeoff visible. In the TabPFN
panel, \Cref{fig:prediction-gain-map}(b), a mask must reach recall
above roughly $0.8$
before it is predicted to beat the all-feature baseline. This is a hard
training regime. Without an explicit penalty for false positives or
mask size, a learner can collapse to selecting nearly all variables.
The shape of the region is precisely why exact F1 against $B$ is not
the right objective. F1 scores every error alike, whereas the
prediction gain map distinguishes errors by their predictive cost. It also shows that
the downstream regressor decides whether a blanket helps.
This is the bridge from diagnosis to the research directions of
\Cref{sec:implications}.

\subsection{Beyond minimality}
\label{sec:beyond-min}

Layered blankets and the prediction gain map converge on the same
revision of the target. \Cref{thm:mb-bayes} is a population statement:
$B(Y)$ is the unique smallest feature set that reproduces
$\E[Y\mid\X]$ exactly. Prediction at finite $n$ asks a different
question. By \Cref{eq:risk-gap} a mask pays for the boundary
information it misses, but it is repaid in estimation variance for
every column it drops, so a mask that is \emph{not} Bayes sufficient
can still beat the full feature set. That trade is what
\Cref{sec:reward-map} measures, and \Cref{prop:superset} shows it is
safe in one direction: any superset of $B(Y)$ leaves the population
conditional intact. The prediction-aligned object is therefore not the
singleton $B(Y)$ but a neighborhood around it.

\Cref{fig:boundary-layers} sketches that neighborhood. The exact
boundary is the innermost set. The second layer, mixed-layer masks,
and an outer shell of proxy variables all sit close by, and each
trades a little finite-sample efficiency for robustness to recall
errors. The shell is not noise. In causal inference, proxy variables
outside the minimal boundary carry information about latent or
unmeasured causes, and are valuable precisely when boundary variables
are missing or noisy
\citep{10.1093/biomet/asy038,NEURIPS2021_dcf32197}. The operative
question is no longer ``did we recover $B(Y)$?'' but ``which set in
this neighborhood does a given regressor want?'' \Cref{sec:implications}
turns this question into concrete research directions.

\begin{figure*}[htp]
  \centering
  \includegraphics[width=0.92\textwidth]{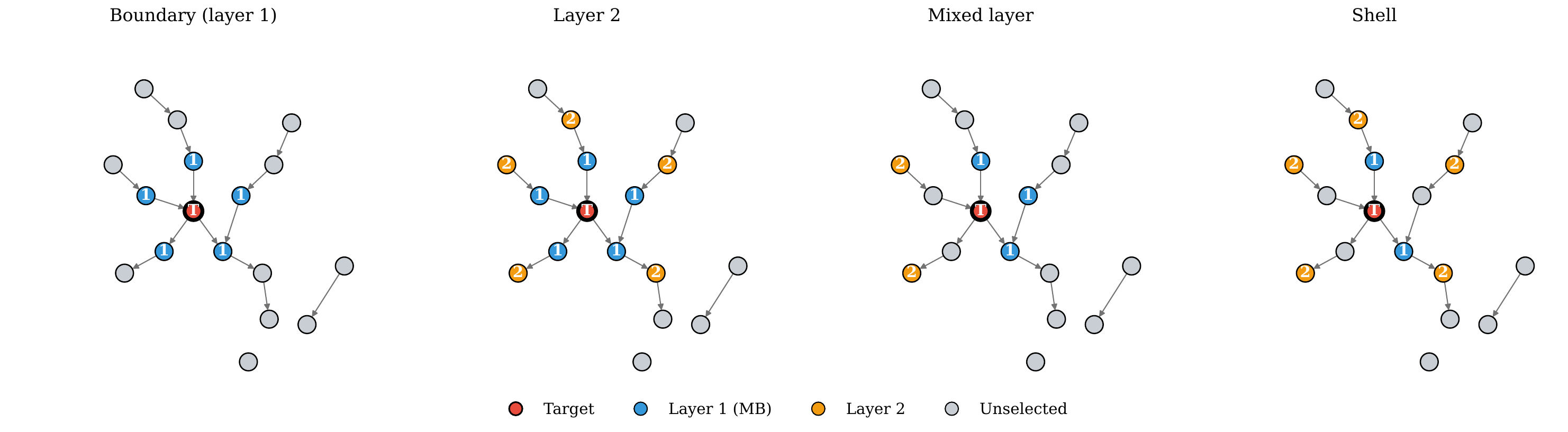}
  \caption{The exact
    boundary is the minimal sufficient set, but layer-2, mixed-layer,
    and shell masks can retain useful redundant or proxy variables.
    This distinction matters when prediction rewards recall more than
    exact minimality.}
  \Description{A four-panel DAG schematic showing the exact Markov
    boundary, the second boundary layer, a mixed-layer mask, and an
    outer shell around the target.}
  \label{fig:boundary-layers}
\end{figure*}

\section{Implications}
\label{sec:implications}

The three failure mechanisms of \Cref{sec:ugly} are not dead
ends.
Each points to a way forward, and both directions below reuse a
resource \SCMBench{} made explicit. A synthetic SCM prior carries
the ground-truth boundary $B(Y)$ alongside the data, a supervision
signal that real observational tables never provide.

\subsection{Scaling Markov-boundary estimation}
\label{sec:impl-scale}

The scalability failure of \Cref{sec:ugly-scale} is structural. A
constraint-based estimator pays an $O(F^d)$ conditional-independence
search for every new dataset. That cost is unavoidable only if the
search is repeated per dataset, and it need not be. Tabular foundation
models such as TabPFN and TabICL are already pre-trained on millions
of synthetic tasks
\citep{hollmann2022tabpfn,qu2025tabicl,muller2025position}. When those
tasks are generated from SCMs, each one carries not just a table and a
target but the target's Markov boundary and the generating graph. The
boundary is free supervision sitting unused in the prior.

This suggests pre-training a tabular model with a boundary-prediction
head alongside the regression head, under a joint objective
\begin{equation}
\mathcal{L}
  =
  \mathcal{L}_{\mathrm{pred}}(\hat y, y)
  + \lambda\,\mathcal{L}_{\mathrm{mask}}(\hat m, B).
\label{eq:joint-pretrain}
\end{equation}
At inference the model returns a prediction together with an amortized
boundary estimate, and the combinatorial search has been paid once,
during pre-training, rather than once per dataset. Causal foundation
models already make the analogous move for causal-effect estimation,
training on SCM-generated tasks and reusing the learned prior at test
time
\citep{robertson2025dopfnincontextlearningcausal,balazadeh2026causalpfn,
ma2025foundation}. The architecture is the same, and only the
supervision target changes. The asymmetry of \Cref{sec:ugly-asym}
should be written directly into $\mathcal{L}_{\mathrm{mask}}$ as a
recall-weighted penalty, so the estimator is discouraged from dropping
boundary features in the first place.

\subsection{Synergizing blanket and prediction}
\label{sec:impl-colearn}

Scaling the estimator does not by itself remove the objective mismatch
of \Cref{sec:ugly-asym}. Structural recovery and prediction remain
different targets. The deeper fix is to stop treating boundary
discovery as an unsupervised pre-processing step and let prediction
loss inform it. Treating the mask $m$ as a latent variable gives the
factorization
\begin{equation}
P(y, m \mid D) = P(y \mid m, D)\,P(m \mid D),
\label{eq:colearn}
\end{equation}
in which the predictor and the mask are learned together rather than
in sequence.

The content of \Cref{eq:colearn} is not the factorization, which is
just the chain rule, but the objective it licenses. Ordinary
supervised selection solves $\min_S\widehat R(S)$; the direction we
propose is
\begin{equation}
\min_S\ \widehat R(S)+\lambda\,\Omega(S;B),
\label{eq:structural-selection}
\end{equation}
with $\widehat R$ the prediction loss and $\Omega$ a structural prior
rewarding preservation of the sufficiency structure of $B$. This is
the missing middle of \Cref{fig:trinity}: implicit selectors optimize
$\widehat R$ and ignore $\Omega$, so nothing protects sufficiency,
while boundary discovery optimizes $\Omega$ blind to $\widehat R$, so
nothing tells it a false negative costs more than a false positive.
Causal structure does not identify the finite-sample optimum, as our
results show; it anchors sufficiency while prediction loss decides how
far to depart from minimality.

\section{Related Work}
\label{sec:related-work}

\noindent\textbf{Markov blanket discovery.}
Markov blanket discovery can be viewed as the local version of causal
discovery. Full-graph methods recover a DAG, CPDAG, or equivalence
class and then derive the target blanket from parents, children, and
spouses. This connects blanket recovery to score-based search, such as
GES \citep{10.1162/153244303321897717}, and to constraint-based causal
discovery, where conditional-independence tests determine the graph up
to Markov equivalence
\citep{spirtes2000causation,verma1990equivalence,
meek1995causal}. These methods are designed primarily for structure.
Their natural outputs are edges, equivalence classes, separating sets,
or local neighborhoods.

Local blanket algorithms avoid full DAG recovery by searching directly
around the target. Grow-Shrink expands and prunes a candidate blanket
through conditional-independence tests
\citep{NIPS1999_5d79099f}. HITON-MB and later local
methods use related tests to identify parents, children, and spouses
\citep{aliferis2003hiton,tsamardinos2003algorithms,
tsamardinos2006max,aliferis2010local}. Our study uses this literature
as the natural baseline for the ``estimate blanket, then predict''
pipeline. The difference is the evaluation target. We ask whether the
estimated blanket improves downstream regression, not whether it
maximizes exact structural recovery.

\noindent\textbf{Feature selection.}
The Markov blanket is also a classical target for feature selection.
Under the usual graphical assumptions, it is the minimal feature set
that preserves the target conditional, so it removes both irrelevant
variables and variables made redundant by the boundary. Causality-based
feature-selection work uses this idea to connect predictive
parsimony, robustness, and interpretability
\citep{yu2020causality,peters2017elements}. Recent
work~\citep{yin2024integrating} extends Markov blankets to
representation learning for domain generalization. In this view, causal
structure is useful not only because it explains the data-generating
process, but because it provides a principled feature mask for a
supervised learner.

Our results refine that feature-selection story. The oracle boundary
does improve prediction in wide, redundant regimes, but the
advantage is finite-sample and regressor-dependent. Moreover, exact
boundary identification is not the only goal relevant for prediction.
False negatives and false positives have different costs, and many
over-inclusive masks can preserve most of the oracle gain. This places
our work between causal feature selection and empirical model
selection. We keep the Markov boundary as the reference object, but we
evaluate masks by prediction loss.

\noindent\textbf{Tabular foundation models.}
Tabular foundation models use synthetic task priors to train predictors
that can be reused across tabular datasets. TabPFN demonstrates this
idea for small tabular tasks
\citep{hollmann2022tabpfn,muller2025position}, while TabICL
extends the in-context-learning framing to larger tabular settings
\citep{qu2025tabicl,qu2026tabiclv2}. In this paper,
these models are downstream regressors. They are not the main method
or the main narrative device. Their role is useful precisely because
they test whether the MB gap persists for modern tabular predictors
that already encode strong prior information.

Causal foundation models make a parallel move for causal inference.
They train on SCM-generated tasks and use the learned prior at test
time for causal queries
\citep{robertson2025dopfnincontextlearningcausal,balazadeh2026causalpfn,
ma2025foundation}. This line of work is relevant because SCM
priors contain ground-truth blankets in addition to samples and graph
structure. We do not propose a causal or tabular foundation model here.
Instead, our results suggest a future training paradigm. Use the
blanket information available in SCM priors to learn
prediction-aligned feature masks jointly with prediction.

\section{Conclusion}
\label{sec:conclusion}

Our main evidence is simulated. \SCMBench{} is broad across
feature
counts, graph densities, and six SCM families, but real-world
validation remains necessary~\citep{pmlr-v275-brouillard25a}. Good candidates include ARTH150, a
107-node Gaussian linear network
\citep{scutari2010learning,opgen2007correlation}, and
causalAssembly, a 98-node industrial assembly-line process
\citep{pmlr-v236-gobler24a}. DREAM4 size-100 and DREAM5
in-silico Net 1 provide continuous gene-regulatory networks
\citep{doi:10.1073/pnas.0913357107,marbach2012wisdom}.
SynTReN provides configurable continuous subnetworks from real
transcriptional networks \citep{van2006syntren}.
A second limitation is that regressors are
evaluated under a fixed protocol rather than fully fine-tuned per
dataset or cohort. A third limitation is scope. We study regression
with RMSE/MSE losses, not classification. Classification may interact
with boundary size, redundancy, and mask errors differently.

Markov boundaries are useful for understanding when feature parsimony
improves tabular prediction. The exact boundary is the minimal
population-sufficient feature set, and the oracle gap is real in the
high-dimensional redundant regime. But exact unsupervised boundary
recovery is too narrow as a prediction objective. Prediction needs
masks that preserve boundary information, control redundant features
according to the downstream regressor, and scale to the regime where
the oracle gap is visible.
The central lesson is therefore not that every predictor should recover the exact Markov boundary. Instead, the boundary should guide the objective, not determine the output. The open center of \Cref{fig:trinity} is filled by prediction loss with regularization toward population sufficiency, rather than by improved exact recovery.

\begin{acks}
  This research is funded by NSF grants \#2311716 “CausalBench: A Cyberinfrastructure for Causal-Learning Benchmarking for Efficacy, Reproducibility, and Scientific Collaboration”, \#2435886 “PanAX: Accelerating Epidemic Science through Causally-Informed Domain Generalization and Knowledge Transfer”, \#2622265. “ NSF Center for Analysis and Prediction of Pandemic Expansion (APPEX)” and by US Army Corps of Engineers "Engineering With Nature Initiative" through Cooperative Ecosystem Studies Unit Agreement \#W912HZ-21-2-0040
\end{acks}
\section{GenAI Disclosure}
AI Agents have been used in this work for coding, and for editorial support.
All AI-assisted text and code was reviewed and edited by the authors.

\bibliographystyle{ACM-Reference-Format}
\bibliography{references}

\end{document}

%% file: figs/src/trinity.tikz
%
\begin{tikzpicture}
  \coordinate (top)   at (0,2.3);
  \coordinate (left)  at (-1.992,-1.15);
  \coordinate (right) at (1.992,-1.15);
  \draw[line width=0.9pt, color=hairline]
    (top) -- (left) -- (right) -- cycle;
  \fill[bodyink] (top) circle (2pt);
  \fill[bodyink] (left) circle (2pt);
  \fill[bodyink] (right) circle (2pt);
  \node[bodyink, font=\bfseries, anchor=south]
    at (0,2.46) {Sufficiency};
  \node[bodyink, font=\bfseries, anchor=north]
    at (-1.992,-1.34) {Scalability};
  \node[bodyink, font=\bfseries, anchor=north]
    at (1.992,-1.34) {Minimality};
  \node[catOracle, font=\footnotesize\bfseries, align=center,
    fill=insightfill, inner sep=2pt]
    at (-0.996,0.575) {Tabular\\foundation\\models};
  \node[catGood, font=\footnotesize\bfseries, align=center,
    fill=insightfill, inner sep=2pt]
    at (0.996,0.575) {Markov\\boundary\\discovery};
  \node[catBad, font=\footnotesize\bfseries, align=center,
    fill=insightfill, inner sep=2pt]
    at (0,-1.15) {Implicit\\feature\\selection};
  \node[mutedtext, font=\itshape\footnotesize, align=center]
    at (0,-0.1) {Open design space};
\end{tikzpicture}%

%% file: arxiv.bbl

\begin{thebibliography}{43}


\ifx \showCODEN    \undefined \def \showCODEN     #1{\unskip}     \fi
\ifx \showISBNx    \undefined \def \showISBNx     #1{\unskip}     \fi
\ifx \showISBNxiii \undefined \def \showISBNxiii  #1{\unskip}     \fi
\ifx \showISSN     \undefined \def \showISSN      #1{\unskip}     \fi
\ifx \showLCCN     \undefined \def \showLCCN      #1{\unskip}     \fi
\ifx \shownote     \undefined \def \shownote      #1{#1}          \fi
\ifx \showarticletitle \undefined \def \showarticletitle #1{#1}   \fi
\ifx \showURL      \undefined \def \showURL       {\relax}        \fi
\providecommand\bibfield[2]{#2}
\providecommand\bibinfo[2]{#2}
\providecommand\natexlab[1]{#1}
\providecommand\showeprint[2][]{arXiv:#2}

\bibitem[Aliferis et~al\mbox{.}(2010)]%
        {aliferis2010local}
\bibfield{author}{\bibinfo{person}{Constantin~F Aliferis}, \bibinfo{person}{Alexander Statnikov}, \bibinfo{person}{Ioannis Tsamardinos}, \bibinfo{person}{Subramani Mani}, {and} \bibinfo{person}{Xenofon~D Koutsoukos}.} \bibinfo{year}{2010}\natexlab{}.
\newblock \showarticletitle{Local causal and {Markov} blanket induction for causal discovery and feature selection for classification part {I}: Algorithms and empirical evaluation}.
\newblock \bibinfo{journal}{\emph{Journal of Machine Learning Research}}  \bibinfo{volume}{11} (\bibinfo{year}{2010}), \bibinfo{pages}{171--234}.
\newblock


\bibitem[Aliferis et~al\mbox{.}(2003)]%
        {aliferis2003hiton}
\bibfield{author}{\bibinfo{person}{Constantin~F Aliferis}, \bibinfo{person}{Ioannis Tsamardinos}, {and} \bibinfo{person}{Alexander Statnikov}.} \bibinfo{year}{2003}\natexlab{}.
\newblock \showarticletitle{HITON: a novel Markov Blanket algorithm for optimal variable selection}. In \bibinfo{booktitle}{\emph{AMIA annual symposium proceedings}}, Vol.~\bibinfo{volume}{2003}. \bibinfo{pages}{21}.
\newblock


\bibitem[Balazadeh et~al\mbox{.}(2026)]%
        {balazadeh2026causalpfn}
\bibfield{author}{\bibinfo{person}{Vahid Balazadeh}, \bibinfo{person}{Hamidreza Kamkari}, \bibinfo{person}{Valentin Thomas}, \bibinfo{person}{Junwei Ma}, \bibinfo{person}{Bingru Li}, \bibinfo{person}{Jesse~C. Cresswell}, {and} \bibinfo{person}{Rahul Krishnan}.} \bibinfo{year}{2026}\natexlab{}.
\newblock \showarticletitle{Causal{PFN}: Amortized Causal Effect Estimation via In-Context Learning}. In \bibinfo{booktitle}{\emph{The Thirty-ninth Annual Conference on Neural Information Processing Systems}}.
\newblock
\urldef\tempurl%
\url{https://openreview.net/forum?id=RblaNJGx8C}
\showURL{%
\tempurl}


\bibitem[Brouillard et~al\mbox{.}(2025)]%
        {pmlr-v275-brouillard25a}
\bibfield{author}{\bibinfo{person}{Philippe Brouillard}, \bibinfo{person}{Chandler Squires}, \bibinfo{person}{Jonas Wahl}, \bibinfo{person}{Konrad K"{o}rding}, \bibinfo{person}{Karen Sachs}, \bibinfo{person}{Alexandre Drouin}, {and} \bibinfo{person}{Dhanya Sridhar}.} \bibinfo{year}{2025}\natexlab{}.
\newblock \showarticletitle{The Landscape of Causal Discovery Data: Grounding Causal Discovery in Real-World Applications}. In \bibinfo{booktitle}{\emph{Proceedings of the Fourth Conference on Causal Learning and Reasoning}} \emph{(\bibinfo{series}{Proceedings of Machine Learning Research}, Vol.~\bibinfo{volume}{275})}, \bibfield{editor}{\bibinfo{person}{Biwei Huang} {and} \bibinfo{person}{Mathias Drton}} (Eds.). \bibinfo{publisher}{PMLR}, \bibinfo{pages}{834--873}.
\newblock
\urldef\tempurl%
\url{https://proceedings.mlr.press/v275/brouillard25a.html}
\showURL{%
\tempurl}


\bibitem[Chen and Guestrin(2016)]%
        {10.1145/2939672.2939785}
\bibfield{author}{\bibinfo{person}{Tianqi Chen} {and} \bibinfo{person}{Carlos Guestrin}.} \bibinfo{year}{2016}\natexlab{}.
\newblock \showarticletitle{XGBoost: A Scalable Tree Boosting System}. In \bibinfo{booktitle}{\emph{Proceedings of the 22nd ACM SIGKDD International Conference on Knowledge Discovery and Data Mining}} (San Francisco, California, USA) \emph{(\bibinfo{series}{KDD '16})}. \bibinfo{publisher}{Association for Computing Machinery}, \bibinfo{address}{New York, NY, USA}, \bibinfo{pages}{785–794}.
\newblock
\showISBNx{9781450342322}
\href{https://doi.org/10.1145/2939672.2939785}{doi:\nolinkurl{10.1145/2939672.2939785}}


\bibitem[Chickering(2003)]%
        {10.1162/153244303321897717}
\bibfield{author}{\bibinfo{person}{David~Maxwell Chickering}.} \bibinfo{year}{2003}\natexlab{}.
\newblock \showarticletitle{Optimal structure identification with greedy search}.
\newblock \bibinfo{journal}{\emph{J. Mach. Learn. Res.}} \bibinfo{volume}{3}, \bibinfo{number}{null} (\bibinfo{date}{March} \bibinfo{year}{2003}), \bibinfo{pages}{507–554}.
\newblock
\showISSN{1532-4435}


\bibitem[Erd{\H{o}}s and R{\'e}nyi(1960)]%
        {erdos1960evolution}
\bibfield{author}{\bibinfo{person}{Paul Erd{\H{o}}s} {and} \bibinfo{person}{Alfr{\'e}d R{\'e}nyi}.} \bibinfo{year}{1960}\natexlab{}.
\newblock \showarticletitle{On the evolution of random graphs}.
\newblock \bibinfo{journal}{\emph{Publications of the Mathematical Institute of the Hungarian Academy of Sciences}}  \bibinfo{volume}{5} (\bibinfo{year}{1960}), \bibinfo{pages}{17--61}.
\newblock


\bibitem[G\"obler et~al\mbox{.}(2024)]%
        {pmlr-v236-gobler24a}
\bibfield{author}{\bibinfo{person}{Konstantin G\"obler}, \bibinfo{person}{Tobias Windisch}, \bibinfo{person}{Mathias Drton}, \bibinfo{person}{Tim Pychynski}, \bibinfo{person}{Martin Roth}, {and} \bibinfo{person}{Steffen Sonntag}.} \bibinfo{year}{2024}\natexlab{}.
\newblock \showarticletitle{$\texttt{causalAssembly}$: Generating Realistic Production Data for Benchmarking Causal Discovery}. In \bibinfo{booktitle}{\emph{Proceedings of the Third Conference on Causal Learning and Reasoning}} \emph{(\bibinfo{series}{Proceedings of Machine Learning Research}, Vol.~\bibinfo{volume}{236})}, \bibfield{editor}{\bibinfo{person}{Francesco Locatello} {and} \bibinfo{person}{Vanessa Didelez}} (Eds.). \bibinfo{publisher}{PMLR}, \bibinfo{pages}{609--642}.
\newblock
\urldef\tempurl%
\url{https://proceedings.mlr.press/v236/gobler24a.html}
\showURL{%
\tempurl}


\bibitem[Hastie et~al\mbox{.}(2009)]%
        {hastie2009elements}
\bibfield{author}{\bibinfo{person}{Trevor Hastie}, \bibinfo{person}{Robert Tibshirani}, {and} \bibinfo{person}{Jerome Friedman}.} \bibinfo{year}{2009}\natexlab{}.
\newblock \bibinfo{booktitle}{\emph{The Elements of Statistical Learning: Data Mining, Inference, and Prediction} (\bibinfo{edition}{2} ed.)}.
\newblock \bibinfo{publisher}{Springer}, \bibinfo{address}{New York, NY}.
\newblock
\showISBNx{978-0-387-84857-0}
\href{https://doi.org/10.1007/978-0-387-84858-7}{doi:\nolinkurl{10.1007/978-0-387-84858-7}}


\bibitem[Hoerl and Kennard(1970)]%
        {Hoerl01021970}
\bibfield{author}{\bibinfo{person}{Arthur~E. Hoerl} {and} \bibinfo{person}{Robert~W. Kennard}.} \bibinfo{year}{1970}\natexlab{}.
\newblock \showarticletitle{Ridge Regression: Biased Estimation for Nonorthogonal Problems}.
\newblock \bibinfo{journal}{\emph{Technometrics}} \bibinfo{volume}{12}, \bibinfo{number}{1} (\bibinfo{year}{1970}), \bibinfo{pages}{55--67}.
\newblock
\showeprint{https://doi.org/10.1080/00401706.1970.10488634}
\href{https://doi.org/10.1080/00401706.1970.10488634}{doi:\nolinkurl{10.1080/00401706.1970.10488634}}


\bibitem[Hollmann et~al\mbox{.}(2023)]%
        {hollmann2022tabpfn}
\bibfield{author}{\bibinfo{person}{Noah Hollmann}, \bibinfo{person}{Samuel M{\"u}ller}, \bibinfo{person}{Katharina Eggensperger}, {and} \bibinfo{person}{Frank Hutter}.} \bibinfo{year}{2023}\natexlab{}.
\newblock \showarticletitle{Tab{PFN}: A Transformer That Solves Small Tabular Classification Problems in a Second}. In \bibinfo{booktitle}{\emph{The Eleventh International Conference on Learning Representations}}.
\newblock
\urldef\tempurl%
\url{https://openreview.net/forum?id=cp5PvcI6w8_}
\showURL{%
\tempurl}


\bibitem[Hornik et~al\mbox{.}(1989)]%
        {HORNIK1989359}
\bibfield{author}{\bibinfo{person}{Kurt Hornik}, \bibinfo{person}{Maxwell Stinchcombe}, {and} \bibinfo{person}{Halbert White}.} \bibinfo{year}{1989}\natexlab{}.
\newblock \showarticletitle{Multilayer feedforward networks are universal approximators}.
\newblock \bibinfo{journal}{\emph{Neural Networks}} \bibinfo{volume}{2}, \bibinfo{number}{5} (\bibinfo{year}{1989}), \bibinfo{pages}{359--366}.
\newblock
\showISSN{0893-6080}
\href{https://doi.org/10.1016/0893-6080(89)90020-8}{doi:\nolinkurl{10.1016/0893-6080(89)90020-8}}


\bibitem[Jennewein et~al\mbox{.}(2023)]%
        {sol}
\bibfield{author}{\bibinfo{person}{Douglas~M. Jennewein}, \bibinfo{person}{Johnathan Lee}, \bibinfo{person}{Chris Kurtz}, \bibinfo{person}{William Dizon}, \bibinfo{person}{Ian Shaeffer}, \bibinfo{person}{Alan Chapman}, \bibinfo{person}{Alejandro Chiquete}, \bibinfo{person}{Josh Burks}, \bibinfo{person}{Amber Carlson}, \bibinfo{person}{Natalie Mason}, \bibinfo{person}{Arhat Kobwala}, \bibinfo{person}{Thirugnanam Jagadeesan}, \bibinfo{person}{Pratiksha~Sharma Basani}, \bibinfo{person}{Torey Battelle}, \bibinfo{person}{Rebecca Belshe}, \bibinfo{person}{Debra McCaffrey}, \bibinfo{person}{Marisa Brazil}, \bibinfo{person}{Chaitanya Inouye}, \bibinfo{person}{Rajiv Kalia}, \bibinfo{person}{Ken Le}, \bibinfo{person}{Christopher Rylko}, {and} \bibinfo{person}{Steve Ruben}.} \bibinfo{year}{2023}\natexlab{}.
\newblock \showarticletitle{The {Sol} Supercomputer at {Arizona State University}}. In \bibinfo{booktitle}{\emph{Practice and Experience in Advanced Research Computing (PEARC '23)}}. \bibinfo{publisher}{ACM}, \bibinfo{pages}{296--301}.
\newblock
\href{https://doi.org/10.1145/3569951.3597573}{doi:\nolinkurl{10.1145/3569951.3597573}}


\bibitem[Koller and Friedman(2009)]%
        {10.5555/1795555}
\bibfield{author}{\bibinfo{person}{Daphne Koller} {and} \bibinfo{person}{Nir Friedman}.} \bibinfo{year}{2009}\natexlab{}.
\newblock \bibinfo{booktitle}{\emph{Probabilistic Graphical Models: Principles and Techniques - Adaptive Computation and Machine Learning}}.
\newblock \bibinfo{publisher}{The MIT Press}.
\newblock
\showISBNx{0262013193}


\bibitem[Laird and Ware(1982)]%
        {c0ae3670-c51a-3c43-9f8f-601a49c19723}
\bibfield{author}{\bibinfo{person}{Nan~M. Laird} {and} \bibinfo{person}{James~H. Ware}.} \bibinfo{year}{1982}\natexlab{}.
\newblock \showarticletitle{Random-Effects Models for Longitudinal Data}.
\newblock \bibinfo{journal}{\emph{Biometrics}} \bibinfo{volume}{38}, \bibinfo{number}{4} (\bibinfo{year}{1982}), \bibinfo{pages}{963--974}.
\newblock
\showISSN{0006341X, 15410420}
\urldef\tempurl%
\url{http://www.jstor.org/stable/2529876}
\showURL{%
\tempurl}


\bibitem[Ma et~al\mbox{.}(2026)]%
        {ma2025foundation}
\bibfield{author}{\bibinfo{person}{Yuchen Ma}, \bibinfo{person}{Dennis Frauen}, \bibinfo{person}{Emil Javurek}, {and} \bibinfo{person}{Stefan Feuerriegel}.} \bibinfo{year}{2026}\natexlab{}.
\newblock \showarticletitle{Foundation Models for Causal Inference via Prior-Data Fitted Networks}. In \bibinfo{booktitle}{\emph{The Fourteenth International Conference on Learning Representations}}.
\newblock
\urldef\tempurl%
\url{https://openreview.net/forum?id=d2L1ndOKjq}
\showURL{%
\tempurl}


\bibitem[Marbach et~al\mbox{.}(2012)]%
        {marbach2012wisdom}
\bibfield{author}{\bibinfo{person}{Daniel Marbach}, \bibinfo{person}{James~C Costello}, \bibinfo{person}{Robert K{\"u}ffner}, \bibinfo{person}{Nicole~M Vega}, \bibinfo{person}{Robert~J Prill}, \bibinfo{person}{Diogo~M Camacho}, \bibinfo{person}{Kyle~R Allison}, \bibinfo{person}{Manolis Kellis}, \bibinfo{person}{James~J Collins}, {et~al\mbox{.}}} \bibinfo{year}{2012}\natexlab{}.
\newblock \showarticletitle{Wisdom of crowds for robust gene network inference}.
\newblock \bibinfo{journal}{\emph{Nature methods}} \bibinfo{volume}{9}, \bibinfo{number}{8} (\bibinfo{year}{2012}), \bibinfo{pages}{796--804}.
\newblock


\bibitem[Marbach et~al\mbox{.}(2010)]%
        {doi:10.1073/pnas.0913357107}
\bibfield{author}{\bibinfo{person}{Daniel Marbach}, \bibinfo{person}{Robert~J. Prill}, \bibinfo{person}{Thomas Schaffter}, \bibinfo{person}{Claudio Mattiussi}, \bibinfo{person}{Dario Floreano}, {and} \bibinfo{person}{Gustavo Stolovitzky}.} \bibinfo{year}{2010}\natexlab{}.
\newblock \showarticletitle{Revealing strengths and weaknesses of methods for gene network inference}.
\newblock \bibinfo{journal}{\emph{Proceedings of the National Academy of Sciences}} \bibinfo{volume}{107}, \bibinfo{number}{14} (\bibinfo{year}{2010}), \bibinfo{pages}{6286--6291}.
\newblock
\showeprint{https://www.pnas.org/doi/pdf/10.1073/pnas.0913357107}
\href{https://doi.org/10.1073/pnas.0913357107}{doi:\nolinkurl{10.1073/pnas.0913357107}}


\bibitem[Margaritis and Thrun(1999)]%
        {NIPS1999_5d79099f}
\bibfield{author}{\bibinfo{person}{Dimitris Margaritis} {and} \bibinfo{person}{Sebastian Thrun}.} \bibinfo{year}{1999}\natexlab{}.
\newblock \showarticletitle{Bayesian Network Induction via Local Neighborhoods}. In \bibinfo{booktitle}{\emph{Advances in Neural Information Processing Systems}}, \bibfield{editor}{\bibinfo{person}{S.~Solla}, \bibinfo{person}{T.~Leen}, {and} \bibinfo{person}{K.~M\"{u}ller}} (Eds.), Vol.~\bibinfo{volume}{12}. \bibinfo{publisher}{MIT Press}.
\newblock
\urldef\tempurl%
\url{https://proceedings.neurips.cc/paper_files/paper/1999/file/5d79099fcdf499f12b79770834c0164a-Paper.pdf}
\showURL{%
\tempurl}


\bibitem[Meek(1995)]%
        {meek1995causal}
\bibfield{author}{\bibinfo{person}{Christopher Meek}.} \bibinfo{year}{1995}\natexlab{}.
\newblock \showarticletitle{Causal inference and causal explanation with background knowledge}.
\newblock \bibinfo{journal}{\emph{Proceedings of the Eleventh Conference on Uncertainty in Artificial Intelligence}} (\bibinfo{year}{1995}), \bibinfo{pages}{403--410}.
\newblock


\bibitem[Miao et~al\mbox{.}(2018)]%
        {10.1093/biomet/asy038}
\bibfield{author}{\bibinfo{person}{Wang Miao}, \bibinfo{person}{Zhi Geng}, {and} \bibinfo{person}{Eric~J Tchetgen~Tchetgen}.} \bibinfo{year}{2018}\natexlab{}.
\newblock \showarticletitle{Identifying causal effects with proxy variables of an unmeasured confounder}.
\newblock \bibinfo{journal}{\emph{Biometrika}} \bibinfo{volume}{105}, \bibinfo{number}{4} (\bibinfo{date}{12} \bibinfo{year}{2018}), \bibinfo{pages}{987--993}.
\newblock
\showISSN{0006-3444}
\showeprint{https://academic.oup.com/biomet/article-pdf/105/4/987/27121264/asy038.pdf}
\href{https://doi.org/10.1093/biomet/asy038}{doi:\nolinkurl{10.1093/biomet/asy038}}


\bibitem[M{\"u}ller et~al\mbox{.}(2022)]%
        {muller2022transformers}
\bibfield{author}{\bibinfo{person}{Samuel M{\"u}ller}, \bibinfo{person}{Noah Hollmann}, \bibinfo{person}{Sebastian Pineda~Arango}, \bibinfo{person}{Josif Grabocka}, {and} \bibinfo{person}{Frank Hutter}.} \bibinfo{year}{2022}\natexlab{}.
\newblock \showarticletitle{Transformers Can Do {B}ayesian Inference}. In \bibinfo{booktitle}{\emph{International Conference on Learning Representations}}.
\newblock
\urldef\tempurl%
\url{https://openreview.net/forum?id=KSugKcbNf9}
\showURL{%
\tempurl}


\bibitem[M{\"u}ller et~al\mbox{.}(2025)]%
        {muller2025position}
\bibfield{author}{\bibinfo{person}{Samuel M{\"u}ller}, \bibinfo{person}{Arik Reuter}, \bibinfo{person}{Noah Hollmann}, \bibinfo{person}{David R{\"u}gamer}, {and} \bibinfo{person}{Frank Hutter}.} \bibinfo{year}{2025}\natexlab{}.
\newblock \showarticletitle{Position: The Future of {B}ayesian Prediction Is Prior-Fitted}. In \bibinfo{booktitle}{\emph{Proceedings of the 42nd International Conference on Machine Learning}} \emph{(\bibinfo{series}{Proceedings of Machine Learning Research}, Vol.~\bibinfo{volume}{267})}. \bibinfo{publisher}{PMLR}, \bibinfo{address}{Vancouver, Canada}.
\newblock
\urldef\tempurl%
\url{https://proceedings.mlr.press/v267/muller25d.html}
\showURL{%
\tempurl}


\bibitem[Opgen-Rhein and Strimmer(2007)]%
        {opgen2007correlation}
\bibfield{author}{\bibinfo{person}{Rainer Opgen-Rhein} {and} \bibinfo{person}{Korbinian Strimmer}.} \bibinfo{year}{2007}\natexlab{}.
\newblock \showarticletitle{From correlation to causation networks: a simple approximate learning algorithm and its application to high-dimensional plant gene expression data}.
\newblock \bibinfo{journal}{\emph{BMC systems biology}} \bibinfo{volume}{1}, \bibinfo{number}{1} (\bibinfo{year}{2007}), \bibinfo{pages}{37}.
\newblock


\bibitem[Pearl(1988)]%
        {10.5555/534975}
\bibfield{author}{\bibinfo{person}{Judea Pearl}.} \bibinfo{year}{1988}\natexlab{}.
\newblock \bibinfo{booktitle}{\emph{Probabilistic Reasoning in Intelligent Systems: Networks of Plausible Inference}}.
\newblock \bibinfo{publisher}{Morgan Kaufmann Publishers Inc.}, \bibinfo{address}{San Francisco, CA, USA}.
\newblock
\showISBNx{1558604790}


\bibitem[Pearl(2009)]%
        {10.5555/1642718}
\bibfield{author}{\bibinfo{person}{Judea Pearl}.} \bibinfo{year}{2009}\natexlab{}.
\newblock \bibinfo{booktitle}{\emph{Causality: Models, Reasoning and Inference} (\bibinfo{edition}{2nd} ed.)}.
\newblock \bibinfo{publisher}{Cambridge University Press}, \bibinfo{address}{USA}.
\newblock
\showISBNx{052189560X}


\bibitem[Peters et~al\mbox{.}(2017)]%
        {peters2017elements}
\bibfield{author}{\bibinfo{person}{Jonas Peters}, \bibinfo{person}{Dominik Janzing}, {and} \bibinfo{person}{Bernhard Sch{\"o}lkopf}.} \bibinfo{year}{2017}\natexlab{}.
\newblock \bibinfo{booktitle}{\emph{Elements of Causal Inference: Foundations and Learning Algorithms}}.
\newblock \bibinfo{publisher}{MIT Press}, \bibinfo{address}{Cambridge, MA}.
\newblock


\bibitem[Qu et~al\mbox{.}(2025)]%
        {qu2025tabicl}
\bibfield{author}{\bibinfo{person}{Jingang Qu}, \bibinfo{person}{David Holzm{\"u}ller}, \bibinfo{person}{Ga{\"e}l Varoquaux}, {and} \bibinfo{person}{Marine Le~Morvan}.} \bibinfo{year}{2025}\natexlab{}.
\newblock \showarticletitle{Tab{ICL}: {A} Tabular Foundation Model for In-Context Learning on Large Data}. In \bibinfo{booktitle}{\emph{Proceedings of the 42nd International Conference on Machine Learning}} \emph{(\bibinfo{series}{Proceedings of Machine Learning Research}, Vol.~\bibinfo{volume}{267})}. \bibinfo{publisher}{PMLR}, \bibinfo{address}{Vancouver, Canada}, \bibinfo{pages}{50817--50847}.
\newblock
\urldef\tempurl%
\url{https://proceedings.mlr.press/v267/qu25d.html}
\showURL{%
\tempurl}


\bibitem[Qu et~al\mbox{.}(2026)]%
        {qu2026tabiclv2}
\bibfield{author}{\bibinfo{person}{Jingang Qu}, \bibinfo{person}{David Holzm{\"u}ller}, \bibinfo{person}{Ga{\"e}l Varoquaux}, {and} \bibinfo{person}{Marine Le~Morvan}.} \bibinfo{year}{2026}\natexlab{}.
\newblock \bibinfo{title}{{TabICLv2}: {A} better, faster, scalable, and open tabular foundation model}.
\newblock
\showeprint[arxiv]{2602.11139}~[cs.LG]
\urldef\tempurl%
\url{https://arxiv.org/abs/2602.11139}
\showURL{%
\tempurl}


\bibitem[Reisach et~al\mbox{.}(2021)]%
        {reisach2021beware}
\bibfield{author}{\bibinfo{person}{Alexander Reisach}, \bibinfo{person}{Christof Seiler}, {and} \bibinfo{person}{Sebastian Weichwald}.} \bibinfo{year}{2021}\natexlab{}.
\newblock \showarticletitle{Beware of the simulated dag! causal discovery benchmarks may be easy to game}.
\newblock   \bibinfo{volume}{34} (\bibinfo{year}{2021}), \bibinfo{pages}{27772--27784}.
\newblock


\bibitem[Reisach et~al\mbox{.}(2023)]%
        {3666122.3666158}
\bibfield{author}{\bibinfo{person}{Alexander~G. Reisach}, \bibinfo{person}{Myriam Tami}, \bibinfo{person}{Christof Seiler}, \bibinfo{person}{Antoine Chambaz}, {and} \bibinfo{person}{Sebastian Weichwald}.} \bibinfo{year}{2023}\natexlab{}.
\newblock \showarticletitle{A scale-invariant sorting criterion to find a causal order in additive noise models}. In \bibinfo{booktitle}{\emph{Proceedings of the 37th International Conference on Neural Information Processing Systems}} (New Orleans, LA, USA) \emph{(\bibinfo{series}{NIPS '23})}. \bibinfo{publisher}{Curran Associates Inc.}, \bibinfo{address}{Red Hook, NY, USA}, Article \bibinfo{articleno}{36}, \bibinfo{numpages}{23}~pages.
\newblock


\bibitem[Robertson et~al\mbox{.}(2025)]%
        {robertson2025dopfnincontextlearningcausal}
\bibfield{author}{\bibinfo{person}{Jake Robertson}, \bibinfo{person}{Arik Reuter}, \bibinfo{person}{Siyuan Guo}, \bibinfo{person}{Noah Hollmann}, \bibinfo{person}{Frank Hutter}, {and} \bibinfo{person}{Bernhard Schölkopf}.} \bibinfo{year}{2025}\natexlab{}.
\newblock \bibinfo{title}{Do-PFN: In-Context Learning for Causal Effect Estimation}.
\newblock
\showeprint[arxiv]{2506.06039}~[cs.LG]
\urldef\tempurl%
\url{https://arxiv.org/abs/2506.06039}
\showURL{%
\tempurl}


\bibitem[Scutari(2010)]%
        {scutari2010learning}
\bibfield{author}{\bibinfo{person}{Marco Scutari}.} \bibinfo{year}{2010}\natexlab{}.
\newblock \showarticletitle{Learning Bayesian networks with the bnlearn R package}.
\newblock \bibinfo{journal}{\emph{Journal of statistical software}}  \bibinfo{volume}{35} (\bibinfo{year}{2010}), \bibinfo{pages}{1--22}.
\newblock


\bibitem[Spirtes et~al\mbox{.}(2000)]%
        {spirtes2000causation}
\bibfield{author}{\bibinfo{person}{Peter Spirtes}, \bibinfo{person}{Clark Glymour}, {and} \bibinfo{person}{Richard Scheines}.} \bibinfo{year}{2000}\natexlab{}.
\newblock \bibinfo{booktitle}{\emph{Causation, Prediction, and Search} (\bibinfo{edition}{2} ed.)}.
\newblock \bibinfo{publisher}{MIT Press}, \bibinfo{address}{Cambridge, MA}.
\newblock


\bibitem[Tibshirani(1996)]%
        {10.1111/j.2517-6161.1996.tb02080.x}
\bibfield{author}{\bibinfo{person}{Robert Tibshirani}.} \bibinfo{year}{1996}\natexlab{}.
\newblock \showarticletitle{Regression Shrinkage and Selection Via the Lasso}.
\newblock \bibinfo{journal}{\emph{Journal of the Royal Statistical Society: Series B (Methodological)}} \bibinfo{volume}{58}, \bibinfo{number}{1} (\bibinfo{date}{01} \bibinfo{year}{1996}), \bibinfo{pages}{267--288}.
\newblock
\showISSN{0035-9246}
\href{https://doi.org/10.1111/j.2517-6161.1996.tb02080.x}{doi:\nolinkurl{10.1111/j.2517-6161.1996.tb02080.x}}


\bibitem[Tsamardinos et~al\mbox{.}(2003)]%
        {tsamardinos2003algorithms}
\bibfield{author}{\bibinfo{person}{Ioannis Tsamardinos}, \bibinfo{person}{Constantin~F. Aliferis}, {and} \bibinfo{person}{Alexander Statnikov}.} \bibinfo{year}{2003}\natexlab{}.
\newblock \showarticletitle{Algorithms for Large Scale {Markov} Blanket Discovery}. In \bibinfo{booktitle}{\emph{Proceedings of the Sixteenth International Florida Artificial Intelligence Research Society Conference (FLAIRS 2003)}}. \bibinfo{publisher}{AAAI Press}, \bibinfo{address}{Menlo Park, CA}, \bibinfo{pages}{376--381}.
\newblock
\urldef\tempurl%
\url{https://aaai.org/papers/flairs-2003-073/}
\showURL{%
\tempurl}


\bibitem[Tsamardinos et~al\mbox{.}(2006)]%
        {tsamardinos2006max}
\bibfield{author}{\bibinfo{person}{Ioannis Tsamardinos}, \bibinfo{person}{Laura~E Brown}, {and} \bibinfo{person}{Constantin~F Aliferis}.} \bibinfo{year}{2006}\natexlab{}.
\newblock \showarticletitle{The max-min hill-climbing {Bayesian} network structure learning algorithm}.
\newblock \bibinfo{journal}{\emph{Machine Learning}} \bibinfo{volume}{65}, \bibinfo{number}{1} (\bibinfo{year}{2006}), \bibinfo{pages}{31--78}.
\newblock


\bibitem[Van~den Bulcke et~al\mbox{.}(2006)]%
        {van2006syntren}
\bibfield{author}{\bibinfo{person}{Tim Van~den Bulcke}, \bibinfo{person}{Koenraad Van~Leemput}, \bibinfo{person}{Bart Naudts}, \bibinfo{person}{Piet van Remortel}, \bibinfo{person}{Hongwu Ma}, \bibinfo{person}{Alain Verschoren}, \bibinfo{person}{Bart De~Moor}, {and} \bibinfo{person}{Kathleen Marchal}.} \bibinfo{year}{2006}\natexlab{}.
\newblock \showarticletitle{SynTReN: a generator of synthetic gene expression data for design and analysis of structure learning algorithms}.
\newblock \bibinfo{journal}{\emph{BMC bioinformatics}} \bibinfo{volume}{7}, \bibinfo{number}{1} (\bibinfo{year}{2006}), \bibinfo{pages}{43}.
\newblock


\bibitem[Verma and Pearl(1990)]%
        {verma1990equivalence}
\bibfield{author}{\bibinfo{person}{Thomas Verma} {and} \bibinfo{person}{Judea Pearl}.} \bibinfo{year}{1990}\natexlab{}.
\newblock \showarticletitle{Equivalence and synthesis of causal models}.
\newblock \bibinfo{journal}{\emph{Proceedings of the Sixth Conference on Uncertainty in Artificial Intelligence}} (\bibinfo{year}{1990}), \bibinfo{pages}{220--227}.
\newblock


\bibitem[Wooldridge(2010)]%
        {0f8094d1-2ebf-3360-90f1-98e50c172f31}
\bibfield{author}{\bibinfo{person}{Jefrey~M. Wooldridge}.} \bibinfo{year}{2010}\natexlab{}.
\newblock \bibinfo{booktitle}{\emph{Econometric Analysis of Cross Section and Panel Data}}.
\newblock \bibinfo{publisher}{The MIT Press}.
\newblock
\showISBNx{9780262232586}
\urldef\tempurl%
\url{http://www.jstor.org/stable/j.ctt5hhcfr}
\showURL{%
\tempurl}


\bibitem[Xu et~al\mbox{.}(2021)]%
        {NEURIPS2021_dcf32197}
\bibfield{author}{\bibinfo{person}{Liyuan Xu}, \bibinfo{person}{Heishiro Kanagawa}, {and} \bibinfo{person}{Arthur Gretton}.} \bibinfo{year}{2021}\natexlab{}.
\newblock \showarticletitle{Deep Proxy Causal Learning and its Application to Confounded Bandit Policy Evaluation}. In \bibinfo{booktitle}{\emph{Advances in Neural Information Processing Systems}}, \bibfield{editor}{\bibinfo{person}{M.~Ranzato}, \bibinfo{person}{A.~Beygelzimer}, \bibinfo{person}{Y.~Dauphin}, \bibinfo{person}{P.S. Liang}, {and} \bibinfo{person}{J.~Wortman Vaughan}} (Eds.), Vol.~\bibinfo{volume}{34}. \bibinfo{publisher}{Curran Associates, Inc.}, \bibinfo{pages}{26264--26275}.
\newblock
\urldef\tempurl%
\url{https://proceedings.neurips.cc/paper_files/paper/2021/file/dcf3219715a7c9cd9286f19db46f2384-Paper.pdf}
\showURL{%
\tempurl}


\bibitem[Yin et~al\mbox{.}(2024)]%
        {yin2024integrating}
\bibfield{author}{\bibinfo{person}{Naiyu Yin}, \bibinfo{person}{Hanjing Wang}, \bibinfo{person}{Yue Yu}, \bibinfo{person}{Tian Gao}, \bibinfo{person}{Amit Dhurandhar}, {and} \bibinfo{person}{Qiang Ji}.} \bibinfo{year}{2024}\natexlab{}.
\newblock \showarticletitle{Integrating Markov blanket discovery into causal representation learning for domain generalization}. In \bibinfo{booktitle}{\emph{European Conference on Computer Vision}}. Springer, \bibinfo{pages}{271--288}.
\newblock


\bibitem[Yu et~al\mbox{.}(2020)]%
        {yu2020causality}
\bibfield{author}{\bibinfo{person}{Kui Yu}, \bibinfo{person}{Xianjie Guo}, \bibinfo{person}{Lin Liu}, \bibinfo{person}{Jiuyong Li}, \bibinfo{person}{Hao Wang}, \bibinfo{person}{Zhaolong Ling}, {and} \bibinfo{person}{Xindong Wu}.} \bibinfo{year}{2020}\natexlab{}.
\newblock \showarticletitle{Causality-based Feature Selection: Methods and Evaluations}.
\newblock \bibinfo{journal}{\emph{Comput. Surveys}} \bibinfo{volume}{53}, \bibinfo{number}{5} (\bibinfo{year}{2020}), \bibinfo{pages}{1--36}.
\newblock
\href{https://doi.org/10.1145/3409382}{doi:\nolinkurl{10.1145/3409382}}


\end{thebibliography}
